\documentclass{article}

 \usepackage[preprint]{neurips_2026}

\usepackage[utf8]{inputenc} 
\usepackage[T1]{fontenc}    
\usepackage{hyperref}       
\usepackage{url}            
\usepackage{booktabs}       
\usepackage{amsfonts}       
\usepackage{nicefrac}       
\usepackage{microtype}      
\usepackage{xcolor}         
\usepackage{graphicx}
\usepackage{algorithm}
\usepackage{algorithmic}
\usepackage{url}
\usepackage{color}
\usepackage{colortbl}
\usepackage{multirow}
\usepackage{multicol}
\usepackage{diagbox}
\usepackage{wrapfig}
\usepackage{makecell}
\usepackage{appendix}
\usepackage{subcaption}
\usepackage{wrapfig}
\usepackage{indentfirst}
\definecolor{tableColor}{HTML}{e9f1f6}

\usepackage{amsmath,amssymb,amsthm}
\usepackage{enumitem}
\usepackage{bm}
\usepackage{mathtools}

\newtheorem{theorem}{Theorem}

\newtheorem{assumption}{Assumption}
\newtheorem{remark}{Remark}
\newcommand{\E}{\mathbb{E}}

\title{HTPO: Towards Exploration-Exploitation Balanced Policy Optimization via Hierarchical Token-level Objective Control}

\author{%
  Xincheng Yao$^{1}$\thanks{Work done during internship at Xiaohongshu.}, Ruoqi Li$^{2}$\thanks{Equal contribution.}, Cheng Chen$^{2}$\thanks{Project lead.}, Daoxin Zhang$^2$, Yi Wu$^2$, Yao Hu$^2$, \\
  \textbf{Chongyang Zhang}$^{1,3}$\thanks{Corresponding Author.} \\
  $^1$School of Information Science and Electronic Engineering, Shanghai Jiao Tong University\\
  $^2$Xiaohongshu Inc.\\
  $^3$MoE Key Lab of Artificial Intelligence, AI Institute, Shanghai Jiao Tong University.\\
  \texttt{\{i-dover, sunny\_zhang\}@sjtu.edu.cn$^1$} \\
  \texttt{\{liruoqi, chengchen7, tangxiaohui, luyun2, xiahou\}@xiaohongshu.com$^2$} \\
}

\begin{document}

\maketitle

\begin{abstract}
  Reinforcement Learning with Verifiable Rewards (RLVR) has emerged as a pivotal technique for enhancing the reasoning capabilities of Large Language Models (LLMs). However, the de facto practice of mainstream RL algorithms is to treat all tokens of one response equally and assign the same optimization objective to each token, failing to provide granular guidance for the reasoning process. While in Chain-of-Thought (CoT) reasoning, different tokens usually play distinct roles (\emph{i.e.}, high-entropy tokens often act as ``forking'' points that can promote exploration, low-entropy tokens tend to execute reasoning steps along the established path, offering minimal benefit to exploration). Therefore, the current RL algorithms lack an effective mechanism to dynamically balance the exploration-exploitation trade-off during learning. To this end, we propose \textbf{H}ierarchical \textbf{T}oken-level Objective Control \textbf{P}olicy \textbf{Op}timization (\textbf{HTPO}), a novel RL algorithm that takes the divide-and-conquer idea to hierarchically partition the response tokens into specific functional groups from three aspects (\emph{i.e.}, prompt difficulty, answer correctness, and token entropy). Within each group, according to the contributions to exploration or exploitation, we design specialized optimization objectives to facilitate the effective execution of each token’s expected functionality. In this way, HTPO can achieve a more balanced exploration-exploitation trade-off. Extensive experiments on challenging reasoning benchmarks validate the superiority of our HTPO algorithm, which significantly outperforms the strong DAPO baseline (\emph{e.g.}, +8.6\% and +6.7\% on AIME'24 and AIME'25, respectively). When scaling test-time compute, the HTPO-trained model maintains a consistent performance advantage over the DAPO baseline, and the gap widens as the sampling budget increases, validating that our adaptive token-level control method fosters effective exploration without sacrificing exploitation performance. Code will be at \url{https://github.com/xcyao00/HTPO}.
\end{abstract}

\section{Introduction}
\label{sec:introduction}

The reasoning capabilities of Large Language Models (LLMs) have evolved substantially, enabling deeper cognitive processes for challenging tasks in mathematics, programming, and open-domain agentic tasks \cite{OpenAI-o1, DeepSeek-R1, Claude-3.7, Kimi-1.5, Qwen3, Tongyi-DeepSearch}. A pivotal technique driving this evolution is large-scale Reinforcement Learning (RL), where models acquire dynamic, multi-step problem-solving strategies, and then consolidate correct cognition and avoid erroneous attempts through RL optimization objectives. In the post-training of LLMs, RL methods have undergone a paradigm shift from the early Proximal Policy Optimization (PPO) \cite{PPO} to the recent Group Relative Policy Optimization (GRPO) \cite{GRPO}, which largely simplifies the optimization landscape by removing the cumbersome reward and value models. Afterwards, numerous GRPO-style and advanced RL works have emerged, stemming from algorithmic innovations \cite{DAPO, VAPO, GSPO, BAPO, SAPO} and counterintuitive empirical insights \cite{DoesRL, RLone}. 

\begin{figure}[ht]
    \centering
    \includegraphics[width=1.0\linewidth]{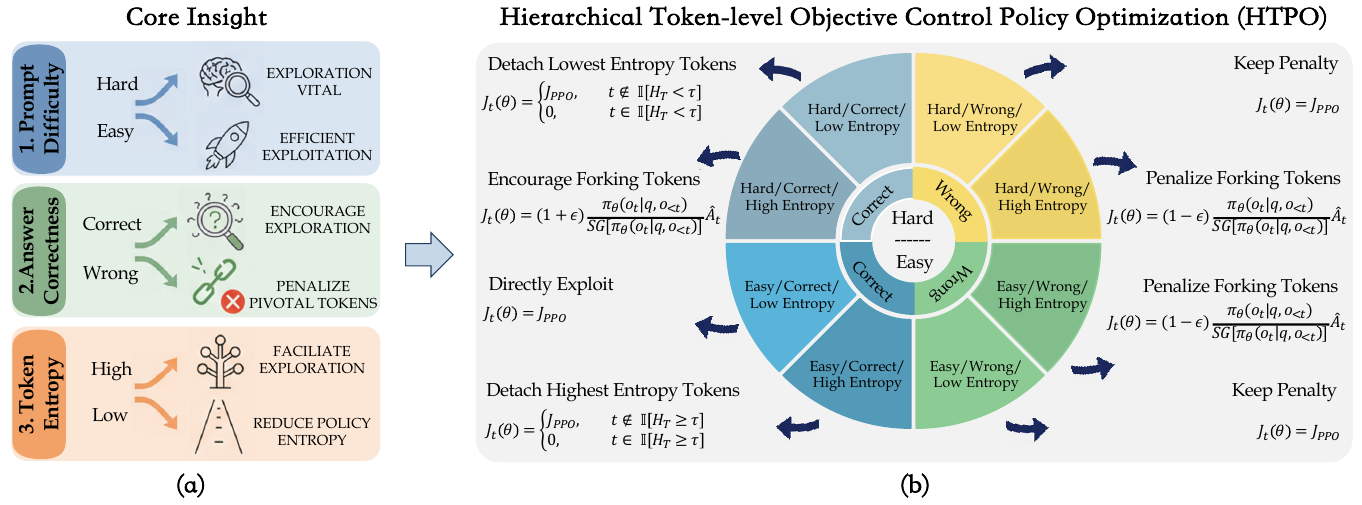}
    \caption{(a) Motivation of our method. Different tokens should not be simply considered as the same, but rather have their respective expected roles according to the prompt difficulty, answer correctness, and token entropy. (b) In HTPO, we hierarchically divide the response tokens into different groups and assign specialized optimization objectives for each group to enhance the efficacy of each token's intended functionality.}
    \label{fig:motivation}
\end{figure}

However, the basic optimization objectives in existing RL implementations still follow the PPO-style clipped loss function (\emph{i.e.}, ${\rm min}\big(r_t(\theta)\hat{A}_t,{\rm clip}(r_t(\theta), 1 - \epsilon, 1 + \epsilon)\hat{A}_t\big)$). A critical issue in PPO and the subsequent RL algorithms is their uniform treatment of all response tokens, assigning an identical optimization objective to each token regardless of its specific role. However, concerning the exploration-exploitation trade-off, different tokens may have distinctive contributions \cite{Entropy-Mechanism, 80/20-rule}, which should be especially discussed in terms of easy or hard problems, as well as correct or wrong answers. These previous approaches neglect the heterogeneous functional roles that tokens play in reasoning processes, potentially hindering further performance gains due to the lack of granular guidance for the reasoning trajectories.

Entropy can serve as a key mechanism to reveal and research the specific roles that different tokens play in CoTs and RL training. One previous work \cite{80/20-rule} has shown that high-entropy tokens are usually ``forking'' tokens that function as pivotal decision points to determine the reasoning trajectory, thereby facilitating reasoning and promoting exploration. In contrast, low-entropy tokens typically act as ``following'' tokens that primarily execute reasoning steps along the established path to guarantee the ongoing linguistic coherence, offering minimal benefit to exploration and potentially impeding it. The authors in \cite{Entropy-Mechanism} also indicate that a token with high advantage and high probability
(low entropy) would reduce policy entropy (causing exploration degradation), whereas a rare token (high entropy) with high advantage would increase policy entropy. Moreover, from the perspective of problem difficulty, different tokens usually play distinct roles as well. For challenging problems, low-certainty (high entropy) tokens will signal valuable exploration; for simpler problems, high-certainty (low entropy) tokens often reflect efficient and reliable reasoning \cite{DACE}. Therefore, to trade off exploration and exploitation, we should also encourage exploration for difficult tasks and take advantage of exploitation for easier ones to enhance learning efficiency. 

In this paper, building upon the above insight, we propose that RL algorithms should eschew the uniform treatment of tokens. The optimal strategy is to distinguish and process tokens separately contingent on their specific roles in CoTs, and then design tailored optimization objectives to more effectively incentivize their roles in reasoning. To this end, we introduce \textbf{H}ierarchical \textbf{T}oken-level Objective Control \textbf{P}olicy \textbf{O}ptimization (\textbf{HTPO}), a novel RL algorithm that takes the divide-and-conquer idea to hierarchically partition the response tokens into eight functional groups from three aspects, including prompt difficulty, answer correctness, and token entropy (see Fig.\ref{fig:motivation}(a)). Regarding prompt difficulty, the core insight is: exploration is vital for hard tasks, whereas exploitation encourages convergence efficiency on easy ones. For answer correctness, the insight is: for correct responses, exploration should still be encouraged to avoid a bias toward over-exploitation; for incorrect responses, pivotal error-inducing tokens should incur more penalties, while neutral tokens should be spared from excessive penalization. With respect to token entropy, the insight is: high-entropy tokens would facilitate policy exploration, whereas learning from low-entropy tokens tends to reduce policy entropy. While each aspect can independently inform token differentiation, our HTPO integrates all three aspects and hierarchically divides the response tokens into eight functional groups. Then, within each group, according to the contributions to exploration or exploitation, we design specialized optimization objectives to facilitate the effective execution of each token's expected functionality (see Fig.\ref{fig:motivation}(b)). For instance, for high-entropy tokens in correct responses to hard prompts, we propose a fixed importance ratio-based non-clipping optimization objective to promote exploration. Conversely, some of the lowest-entropy tokens are detached from the policy gradient to avoid the model's over-exploitation. Our main contributions are as follows:

1. We identify and analyze one key issue behind the previous RL algorithms for LLMs: the identical optimization objective applied to all tokens lacks granular guidance for the reasoning process.

2. We propose HTPO, a novel RL algorithm that hierarchically divides the response tokens from three aspects and assigns tailored optimization objectives to different groups of tokens for achieving a more balanced exploration-exploitation trade-off.

3. We validate HTPO on both the pre-trained and instruction finetuned LLMs and across multiple reasoning benchmarks, demonstrating that it can facilitate effective exploration without sacrificing exploitation performance, delivering significant performance gains over strong baselines. 

\section{Related Work}
\label{sec:related_work}

\textbf{Reinforcement Learning for LLMs.} Before the advent of reasoning-capable models like OpenAI’s o1 \cite{OpenAI-o1}, RL was mainly used to optimize the instruction-following and human preference alignment capabilities of pretrained LLMs \cite{InstructGPT, RLAIF}. PPO \cite{PPO} dominated this era, while powerful, it is notoriously cumbersome, necessitating separate reward and value models that introduce significant computational overhead and training complexity. Therefore, the industry urgently sought a more stable and scalable alternative. Direct Preference Optimization (DPO) \cite{DPO} emerged and quickly marked a milestone by bypassing explicit reward modeling and value model optimization. During this period, we witnessed a proliferation of DPO variants, such as SimPO \cite{SimPO}, KTO \cite{KTO}, ORPO \cite{ORPO}, and BCO \cite{BCO}. However, DPO-like algorithms rely on pre-collected preference data pairs, and due to offline optimization, their performance upper bounds are inferior to online algorithms. Recently, DeepSeek-R1 \cite{DeepSeek-R1} has successfully sparked the popularity of RL with Verifiable Rewards (RLVR) for enhancing reasoning in LLMs, particularly in mathematics, programming, and open-domain agentic tasks \cite{Qwen3, Kimi-k2, GLM-4.5, Claude-4.6, gpt5.2, gemini-3}. DeepSeek-R1 stands out for demonstrating that strong reasoning can emerge only through outcome-based rewards using the online RLVR algorithm like GRPO \cite{GRPO}. Another prominent point is the ``zero RL'' paradigm, where reasoning abilities can be elicited from the base model without conventional SFT cold-start. Afterwards, many open-source works, such as SimpleRLZoo \cite{simplerl-zoo}, and Open-Reasoner-Zero \cite{open-reasoner} have attempted to reproduce DeepSeek-R1. And subsequent RL methods, including DAPO \cite{DAPO}, VAPO \cite{VAPO}, GSPO \cite{GSPO}, BAPO \cite{BAPO}, SAPO \cite{SAPO}, etc, have largely built upon GRPO. In this work, our HTPO is mainly based on a strong baseline, DAPO.

\textbf{Analysis on Reinforcement Learning.} Recent studies have scrutinized the characteristics of RLVR, with a particular focus on the evolution of policy entropy. A key consensus is that not all tokens in CoTs are equally informative, as reasoning behaviors are primarily governed by a minority of ``critical tokens''. In \cite{critical-tokens1, critical-tokens2}, the authors identify ``critical tokens'' as decision points that significantly influence incorrect outcomes, and demonstrate that replacing these tokens can alter model behaviors. Subsequently, entropy has been adopted as a key mechanism for analyzing the dynamics of RLVR. In \cite{80/20-rule}, the authors define ``critical tokens'' through entropy and find that high-entropy minority tokens are closely related to ``critical tokens'' and often act as ``forks'' in the reasoning process. When utilizing only 20\% high-entropy tokens, the performance can be maintained comparable to full-gradient updates, whereas training on 80\% lowest-entropy tokens results in a marked performance decline. In \cite{Entropy-Mechanism}, the authors propose the \emph{Entropy Mechanism} to explain why policy entropy sharply declines in a few training steps. The core observation is that tokens with high advantages and high probabilities would reduce policy entropy, while rare tokens with high advantages would increase entropy. In BAPO \cite{BAPO}, similar findings have been proposed, where tokens that are positive with high probabilities or negative with low probabilities contribute to a reduction in the overall entropy. In this work, dividing tokens based on entropy also draws inspiration from these analytical studies.

\section{Preliminaries}
\label{sec:preliminaries}

The current RL algorithms for post-training LLMs are mostly based on policy gradient, for which the primitive form can be traced back to REINFORCE \cite{REINFORCE}, which can be written as follows:
\begin{equation}
    \nabla(\theta) = \mathbb{E}_{o \sim \pi_{\theta}(\cdot|q)} \bigg[\frac{1}{|o|}\sum_{t=1}^{|o|}\nabla_\theta{\rm log}\pi_\theta(o_t|q,o_{<t})\hat{A}_t
     \bigg].
\end{equation}
where $\pi_\theta$ means policy model, $\hat{A}_t$ represents the advantage value of the $t$-th token. In the LLM RL research, most mainstream RL algorithms \cite{PPO, GRPO, DAPO, GSPO} follow the PPO-style clipped loss function as the optimization objective, which is typically as follows (for each token $o_t$):
\begin{equation}
\label{eq:base_objective}
    J_t(\theta) = 
     {\rm min}\Big(r_t(\theta)\hat{A}_t,{\rm clip}(r_t(\theta), 1 - \epsilon, 1 + \epsilon)\hat{A}_t\Big).
\end{equation}

where $r_t(\theta) = \frac{\pi_\theta(o_t|q,o_{<t})}{\pi_{\theta_{old}}(o_t|q,o_{<t})}$ is the importance ratio for off-policy learning. In PPO \cite{PPO}, each prompt $q$ generates one response, the overall objective is averaged across all tokens: $\mathcal{J}_{PPO}(\theta) = \mathbb{E}_{ o \sim \pi_{\theta_{old}}(\cdot|q)} \Big[\frac{1}{|o|}\sum_{t=1}^{|o|}J_t(\theta)\Big]$. In GRPO \cite{GRPO}, a group of responses $\{o_i\}_{i=1}^G$ will be sampled from the prompt $q$, and the overall objective is averaged across all responses and tokens: $\mathcal{J}_{GRPO}(\theta) = \mathbb{E}_{ \{o_i\}_{i=1}^G \sim \pi_{\theta_{old}}(\cdot|q)} \Big[\frac{1}{G}\sum_{i=1}^{G}\frac{1}{|o_i|}\sum_{t=1}^{|o_i|}J_{i,t}(\theta)\Big]$. The gradient for the above token-wise optimization objective is derived as follows:
\begin{equation}
    \nabla_t(\theta) = 
     \frac{\hat{A}_t}{\pi_{\theta_{old}}(o_t|q,o_{<t})}\nabla_\theta\pi_\theta(o_t|q,o_{<t}).
\end{equation}

For the $t$-th token, the policy gradient in REINFORCE can be rewritten as $\nabla^{RF}_t(\theta) = \frac{\hat{A}_t}{SG[\pi_\theta(o_t|q,o_{<t})]}\nabla_\theta\pi_\theta(o_t|q,o_{<t})$, where $SG[\cdot]$ means \emph{stop-gradient} operation. It can be found that the policy gradient in PPO $\nabla^{PPO}_t(\theta) = \frac{\hat{A}_t}{\pi_{\theta_{old}}(o_t|q,o_{<t})}\nabla_\theta\pi_\theta(o_t|q,o_{<t})$ can be represented as $\nabla^{PPO}_t(\theta) = \frac{\hat{A}_t}{SG[\pi_\theta(o_t|q,o_{<t})]} \cdot \frac{SG[\pi_\theta(o_t|q,o_{<t})]}{\pi_{\theta_{old}}(o_t|q,o_{<t})}\nabla_\theta\pi_\theta(o_t|q,o_{<t}) = SG[r_t(\theta)]\nabla_t^{RF}(\theta)$. Moreover, $\nabla_t^{PPO}(\theta)$ can be rewritten as $\nabla_t^{PPO}(\theta) = SG[r_t(\theta)]\hat{A}_t\nabla_\theta{\rm log}\pi_\theta(o_t|q,o_{<t})$. Therefore, both the importance ratio $r_t(\theta)$ and the advantage $\hat{A}_t$ can be seen as weights for the gradient of model $\pi_\theta$ for predicting $o_t$. If we wish to increase the probability of $o_t$, the $\hat{A}_t$ should be a positive value (\emph{i.e.}, the $r_t(\theta)$ is always positive), and the larger $r_t(\theta)$ and $\hat{A}_t$, the larger the model gradient for increasing the probability of $o_t$. On the contrary, if we wish to reduce the probability of $o_t$, $\hat{A}_t$ should be a negative value, and also the larger $r_t(\theta)$ and $\hat{A}_t$, the larger the gradient for reducing the probability of $o_t$.

In addition, we further point out that in Eq.(\ref{eq:base_objective}), $\hat{A}_t$ is a scalar and the gradient stems from differentiating $r_t(\theta)$. When $r_t(\theta)$ is outside the range $(1-\epsilon, 1+\epsilon)$, the ${\rm clip}$ operator will utilize a scalar to replace $r_t(\theta)$. Therefore, for these out-of-range tokens, they will not participate in optimization. Due to the ${\rm min}$ operator, when $\hat{A}_t > 0$, tokens with $r_t(\theta) > 1+\epsilon$ will not participate in optimization; when $\hat{A}_t < 0$, tokens with $r_t(\theta) < 1-\epsilon$ will be removed out. With all the basic analyses above, it's conducive for us to understand the proposed RL algorithm in the following section.

\section{Method}
\label{sec:method}

The core insight of our method is that since response tokens usually play heterogeneous roles in reasoning processes (to support this statement, we provide a detailed analysis of token entropy patterns in Appendix \ref{sec:token_entropy_analysis}. where the experimental evidence also can support the rationality of our proposed method described below), the optimal strategy is to distinguish tokens into distinct functional groups contingent on their specific roles in CoTs and design tailored optimization objectives that conform to the roles. Specifically, we comprehensively consider three aspects, including prompt difficulty, answer correctness, and token entropy, to hierarchically divide the response tokens into eight functional groups. In the following, we will analyze these token groups and design their corresponding optimization objectives in sequence. Before we begin, it is worth noting that the optimization objectives are not entirely distinct across groups but share commonalities. To facilitate understanding, we still organize the presentation by token groups rather than by objective formulations.

\textbf{Group 1 (hard prompt, correct answer, low-entropy token).} For hard prompts, we should encourage exploration (\emph{i.e.} for finding novel and correct solutions). However, according to the 80/20 rule \cite{80/20-rule}, the majority of tokens are low-entropy, which follow the existing reasoning path and primarily guarantee the ongoing linguistic coherence. These low-entropy tokens tend to reduce policy entropy during training, leading to degraded model exploration. Consequently, constraints should be applied to these tokens to avoid rapid decreases in policy entropy. As indicated in \cite{Entropy-Mechanism}, tokens with higher certainty (\emph{i.e.}, lower entropy) would exert a greater impact on the decrease of policy entropy. Therefore, to impose constraints, we adopt a straightforward strategy: discarding the policy gradients of some lowest-entropy tokens. The token-wise objective is defined as follows:
\begin{equation}
\label{eq:case1_objective}
    J_t(\theta) = \begin{cases}
    {\rm min}\Big(r_t(\theta)\hat{A}_t,{\rm clip}(r_t(\theta), 1 - \epsilon, 1 + \epsilon)\hat{A}_t\Big), & t \notin \mathbb{I}[H_t < \tau_{\rho_1}^\mathcal{B}], \\
    0, & t \in \mathbb{I}[H_t < \tau_{\rho_1}^\mathcal{B}].
    \end{cases}
\end{equation}
where $H_t = -\Sigma_{v\in \mathcal{V}}\pi_{\theta_{old}}(v|q,o_{<t}){\rm log}\pi_{\theta_{old}}(v|q,o_{<t})$ denotes the entropy of the $t$-th token, $\mathbb{I}[\cdot]$ is the indication set that chooses to include $t$ if $H_t < \tau_{\rho_1}^\mathcal{B}$, $\rho_1 \in (0,1]$ is a predefined ratio specifying the proportion of low-entropy tokens to be discarded, and $\tau_{\rho_1}^\mathcal{B}$ is the corresponding entropy threshold within the batch $\mathcal{B}$ ($\mathcal{B}$ represents all tokens in one response, not a batch of prompts). Thus, tokens whose entropy $H_t < \tau_{\rho_1}^\mathcal{B}$ are detached from the policy gradient updates. For token separation, we first sort all tokens by entropy, following the 80/20 rule, the 80\% tokens with lower entropy will be considered as low-entropy tokens, the remaining 20\% tokens are considered as high-entropy tokens. In the subsequent description, we also employ the same way to divide low- and high-entropy tokens.

\textbf{Group 2 (hard prompt, correct answer, high-entropy token).} As evidenced in \cite{80/20-rule}, high-entropy tokens are usually ``forking'' tokens that function as pivotal decision points to determine the reasoning trajectory among multiple potential pathways, thereby facilitating exploration. Moreover, the Entropy Mechanism \cite{Entropy-Mechanism} also indicates that the tokens with high entropy and high advantage would increase policy entropy. Therefore, to guarantee the exploratory capabilities of the model in hard problems, it is crucial to prioritize these high-entropy tokens. As analyzed in Sec.\ref{sec:preliminaries}, enhancing the impact of high-entropy tokens on policy learning can be achieved by either preventing objective clipping or increasing gradient weights. In DAPO \cite{DAPO}, the authors propose to utilize a larger $\epsilon_{high}$ (\emph{e.g.}, 0.28) as the upper clipping bound to ensure more tokens will not be clipped. We extend this by opting to keep all high-entropy tokens from being clipped, but one issue is that excessively large non-clipped $r_t(\theta)$ will cause training instability. To address this, we propose the following fixed importance ratio-based non-clipping optimization objective:
\begin{equation}
\label{eq:case2_objective1}
    J_t(\theta) = (1+\epsilon) \cdot \frac{\pi_\theta(o_t|q,o_{<t})}{SG[\pi_\theta(o_t|q,o_{<t})]} \cdot \hat{A}_t, \quad\quad {\rm if} \; r_t(\theta) > 1 + \epsilon.
\end{equation}
where $SG[\cdot]$ denotes the \emph{stop-gradient} operation. We can see that in this formula, when $r_t(\theta) > 1+\epsilon$, we fix the importance ratio at $1+\epsilon$ and preserve the gradient (\emph{i.e.}, $\pi_\theta(o_t|q,o_{<t})$ is differentiable).

For increasing the gradient weights, also according to the analysis in Sec.\ref{sec:preliminaries}, we can increase the $r_t(\theta)$ or $\hat{A}_t$. In this work, we focus on the importance ratio $r_t(\theta)$\footnote{$\hat{A}_t$ still relies on the outcome advantage assigned to all tokens. Investigating dense credit assignment is beyond the scope of this paper.} and employ the reciprocal of $r_t(\theta)$. For $r_t(\theta) > 1$, the value becomes smaller after being the reciprocal, so only $r_t(\theta) < 1-\epsilon$ becomes to its reciprocal:
\begin{equation}
\label{eq:case2_objective2}
     J_t(\theta) = \frac{\pi_{\theta_{old}}(o_t|q,o_{<t})}{SG[\pi_\theta(o_t|q,o_{<t})]}\cdot \frac{\pi_\theta(o_t|q,o_{<t})}{SG[\pi_\theta(o_t|q,o_{<t})]}  \cdot \hat{A}_t, \quad\quad {\rm if} \; r_t(\theta) < 1 - \epsilon.
\end{equation}
where $\frac{\pi_{\theta_{old}}(o_t|q,o_{<t})}{SG[\pi_\theta(o_t|q,o_{<t})]}$ will also be clipped up to $1 + \epsilon$. 

\textbf{Group 3 (hard prompt, wrong answer, low-entropy token).} In wrong answers, a phenomenon is that one wrong response usually consists of a majority of neutral or correct actions and only a minority of tokens are the root cause of errors \cite{BAPO}. These neutral or correct tokens typically exist in low-entropy tokens and play the role of executing reasoning steps along the established path instead of triggering new reasoning transitions \cite{BAPO}. Therefore, for these tokens, due to their presence in wrong responses, we still need to penalize them, but we decide not to further increase the penalty. Then, we keep the original optimization objective (Eq.(\ref{eq:base_objective})) for these tokens. 

\textbf{Group 4 (hard prompt, wrong answer, high-entropy token).} In the context of wrong answers, high-entropy tokens can still be considered as ``forking'' tokens that are typically the logical connection or transition words to induce erroneous reasoning directions (see more analysis in Appendix \ref{sec:token_entropy_analysis}). Since these tokens introduce reasoning transitions, they are more likely to constitute the root cause of errors \cite{BAPO}. Consequently, we should penalize these tokens more heavily. When $\hat{A}_t < 0$, these tokens are prone to being clipped by the lower clipping bound $1 - \epsilon$. Thus, following the analysis in \textbf{Group 2}, we employ the following fixed importance ratio-based non-clipping optimization format (differently, fixed ratio at $1-\epsilon$):
\begin{equation}
\label{eq:case4_objective}
    J_t(\theta) = (1-\epsilon) \cdot \frac{\pi_\theta(o_t|q,o_{<t})}{SG[\pi_\theta(o_t|q,o_{<t})]} \cdot \hat{A}_t, \quad\quad {\rm if} \; r_t(\theta) < 1 - \epsilon.
\end{equation}

\textbf{Group 5 (easy prompt, correct answer, low-entropy token).} For easy problems, we should encourage exploitation and also don't want unnecessarily long reasoning chains, as more complex reasoning instead may cause the model to overthink and complicate simple problems \cite{DACE}, resulting in errors. Therefore, for reasoning chains that have successfully reached the correct answer, we prioritize direct exploitation over the continued exploration of more complex reasoning paths. Then, we directly learn low-entropy tokens for exploitation, but we also don't want the model to be overly biased to them. To this end, we don't further increase the $r_t(\theta)$ and keep the original optimization objective (Eq.(\ref{eq:base_objective})) for these tokens.

\textbf{Group 6 (easy prompt, correct answer, high-entropy token).} High-entropy tokens often act as logical connectors across sentences, such as ``wait'', ``however'', and ``unless'', etc. For easy problems, we don't want to continue generating more complex reasoning paths. The more logical connector tokens generated by the model, the more likely to lead to complex reasoning chains. Moreover, a phenomenon is that the reasoning model tends to retain the entropy patterns of the base model \cite{80/20-rule} (\emph{i.e.}, after RL, tokens that cause complex reasoning chains may still appear in the answers to easy problems). Therefore, we should reduce the frequency of high-entropy tokens appearing in the responses to easy problems. Analogous to the optimization objective in \textbf{Group 1}, in this case, we opt to clip a fraction of the highest-entropy tokens out from policy gradient updates. We define the token-wise objective as follows:
\begin{equation}
\label{eq:case6_objective}
    J_t(\theta) = \begin{cases}
    {\rm min}\Big(r_t(\theta)\hat{A}_t,{\rm clip}(r_t(\theta), 1 - \epsilon, 1 + \epsilon)\hat{A}_t\Big), & t \notin \mathbb{I}[H_t \geq \tau_{\rho_2}^\mathcal{B}], \\
    0, & t \in \mathbb{I}[H_t \geq \tau_{\rho_2}^\mathcal{B}].
    \end{cases}
\end{equation}
 where $\rho_2 \in (0,1]$ is a predefined ratio specifying the top proportion of high-entropy tokens to be detached, and $\tau_{\rho_2}^\mathcal{B}$ is the corresponding entropy threshold within the batch $\mathcal{B}$.

\textbf{Group 7 (easy prompt, wrong answer, low-entropy token).} As analyzed in \textbf{Group 3}, low-entropy tokens of one wrong response are usually composed of neutral or correct actions that are not the root cause of errors. Since the advantages of these tokens are already negative, following \textbf{Group 3}, we keep the original optimization objective (Eq.(\ref{eq:base_objective})) for these tokens.

\textbf{Group 8 (easy prompt, wrong answer, high-entropy token).} Analogous to the analysis in \textbf{Group 4} and \textbf{Group 6}, high-entropy ``forking'' tokens can still be regarded as having induced wrong reasoning directions, resulting in the final wrong answer. Thus, for easy problems, we don't want too many ``forking'' tokens and should penalize them more heavily. Then, we opt to adopt the token-wise optimization objective (Eq.(\ref{eq:case4_objective})) outlined in \textbf{Group 4}.

Finally, according to the objective formulation, we consolidate the optimization objectives of our HTPO in the following formulation:

\begin{equation}
    J_t(\theta) = \begin{cases}
    0, \quad\quad\quad\quad\quad\quad\quad\quad\quad\quad\quad\quad\quad\quad\quad\quad\;\; o_t \in (\textbf{G1} \land H_t < \tau_{\rho_1}^\mathcal{B}) \; \cup \; (\textbf{G6} \land H_t \geq \tau_{\rho_2}^\mathcal{B}),\\
    
    (1+\epsilon)/(r^{-1}_t) \cdot \frac{\pi_\theta(o_t|q,o_{<t})}{SG[\pi_\theta(o_t|q,o_{<t})]} \cdot \hat{A}_t, \quad\quad\;\; o_t \in \textbf{G2} \land (r_t(\theta) > 1 + \epsilon) / (r_t(\theta) < 1 - \epsilon), \\
    
    (1 - \epsilon) \cdot \frac{\pi_\theta(o_t|q,o_{<t})}{SG[\pi_\theta(o_t|q,o_{<t})]} \cdot \hat{A}_t, \quad\quad\quad\quad\quad\quad\quad\quad\quad  o_t \in (\textbf{G4} \cup \textbf{G8}) \land (r_t(\theta) < 1 - \epsilon), \\
    
    {\rm min}\Big(r_t(\theta)\hat{A}_t,{\rm clip}(r_t(\theta), 1 - \epsilon, 1 + \epsilon)\hat{A}_t\Big), \quad\quad\quad\quad\quad\quad\quad\quad\quad\quad\quad\quad\quad\quad\;\; \textbf{others}.
    \end{cases}
\end{equation}

where $\textbf{G}$ is the abbreviation of \textbf{Group}, $r_t$ denotes $SG[r_t(\theta)]$. The overall algorithmic procedure of our HTPO is presented in Alg.\ref{algorithm1} in Appendix \ref{sec:whole_algorithm}. In Appendix \ref{sec:theory}, we further provide a theoretical analysis of our HTPO (Theorem \ref{thm:consistency}, Inter-Group Gradient Consistency), which states that the group-specific optimization objectives in HTPO do not generate conflicting training signals, and the gradient directions are consistent.

\textbf{Difficulty Estimation.} To quantify prompt difficulty, it is essential to recognize that difficulty is not an absolute property of a prompt but is relative to the policy's current capabilities. Therefore, the difficulty measure for a prompt $q$ should be estimated online with respect to the current policy $\pi_\theta$. To this end, we employ the failure rate as a proxy metric. This provides a practical, policy-related measure of prompt difficulty, which we approximate by calculating the error rate in the group of $G$ responses of each prompt:
\begin{equation}
\label{eq:difficulty_level}
    {\rm diff}(q;\pi_\theta) = 1 - \frac{1}{G}\sum_{i=1}^G\mathbb{I}[{\rm verify}(o_i) = 1].
\end{equation}
A higher ${\rm diff}(q;\pi_\theta)$ value indicates the prompt is more challenging for the current policy model. Accordingly, by introducing a difficulty threshold $\tau_{diff}$, we divide prompts whose ${\rm diff}(q;\pi_\theta) > \tau_{diff}$ as hard prompts, and easy prompts satisfy ${\rm diff}(q;\pi_\theta) \leq \tau_{diff}$.

\section{Experiments}

\subsection{Experimental Setup}
\label{sec:experimental_setup}

\textbf{Datasets and Models.} We employ DAPO-Math-17K \cite{DAPO} as our RL training dataset, owing to its high quality and widespread adoption. For evaluation, we comprehensively compare our method with other state-of-the-art RL methods on five standard mathematical reasoning benchmarks, including AIME'24 \cite{aime24}, AIME'25 \cite{aime25}, AMC'23 \cite{amc23}, Minerva, and OlympiadBench \cite{olympiadbench}. These benchmarks are widely used to assess reasoning capabilities and compare RL algorithms. For the experimental models, we employ the Qwen3-8B series, as this parameter scale is the most commonly adopted and computationally affordable choice in academic papers. To evaluate the robustness of our method, we perform experiments on both the pre-trained Qwen3-8-Base and the instruction-finetuned Qwen3-8B \cite{Qwen3}, representing two distinct application scenarios for RL algorithms.

\textbf{Implementation Details.} We conduct RL training based on VeRL \cite{verl} and follow the recipe of DAPO \cite{DAPO}. For fair comparisons, we employ the same techniques as recommended by DAPO: clip-higher, dynamic sampling, token-level policy gradient loss, and overlong reward shaping. We also align our hyperparameters with DAPO, setting $\epsilon_{low} = 0.2$, $\epsilon_{high} = 0.28$, the maximum response length to 20480, and the cache length to 4096. We use a training batch size of 128 and a mini-batch size of 32, resulting in 4 gradient updates per rollout step. Each prompt generates 16 responses with the decoding temperature of 1.0 and top-p of 1.0. The learning rate and weight decay are set to $1e^{-6}$ and 0.1, respectively. Following DAPO, we also omit both the KL divergence and entropy losses. Regarding HTPO-specific hyperparameters, the default configuration uses the clipping ratios $\rho_1 = 0.006$, $\rho_2 = 0.02$, and the difficulty threshold $\tau_{diff} = 0.5$. For evaluation, we use a temperature of 1.0 and top-p of 0.7. A complete list of hyperparameters is provided in Tab.\ref{tab:hyperparameters} in Appendix \ref{sec:hyperparameters_sup}.

\textbf{Baselines.} We compare HTPO against an improved GRPO \cite{GRPO} baseline with the dynamic sampling trick \cite{DAPO}, as well as the recently proposed strong RL methods. These include: DAPO \cite{DAPO}, a strong improved baseline over GRPO; GSPO \cite{GSPO}, which performs optimization based on the sequence-level importance ratios and clipping; SAPO \cite{SAPO}, which replaces hard clipping with a smooth, temperature-controlled gate to enable more informative and stable updates; 80/20-Rule \cite{80/20-rule}, which performs optimization only on 20\% high-entropy tokens; BAPO \cite{BAPO}, which dynamically adjusts clipping bounds to balance positive
and negative signals. For all methods, we report results based on \textbf{Mean@32}, which can mitigate the impact of sampling randomness.


\subsection{Main Results}

\begin{table*}[ht]
\centering
\caption{Performance comparison on mathematical reasoning benchmarks. All the results are Mean@32 accuracies and averaged over 3 random seeds. GRPO$^\dagger$ is our reimplemented GRPO baseline with the dynamic sampling trick. APG/RPG denotes the absolute/relative performance gain over DAPO. The best scores are marked in \textbf{bold}, and second best scores are \underline{underlined}.}
\label{tab:main_results}
\resizebox{\linewidth}{!} {
\begin{tabular}{ccccccc}
\toprule[0.5mm]
 \textbf{Method} & \textbf{AIME'24} & \textbf{AIME'25} & \textbf{AMC'23} & \textbf{Minerva} & \textbf{OlympiadBench} & \textbf{Avg.} \\
\midrule
\multicolumn{7}{c}{\textbf{\emph{Qwen3-8B-Base}}} \\
\midrule
 GRPO$^\dagger$  & 31.1$\pm$0.60 & 22.9$\pm$0.40 & 70.2$\pm$0.57 & 28.5$\pm$0.22 & 39.8$\pm$0.15 & 38.5 \\
 DAPO & 32.9$\pm$0.52 & 23.7$\pm$0.22 & 75.9$\pm$0.42 & 29.4$\pm$0.09 & \underline{45.7}$\pm$0.15 & 41.5 \\
 
 GSPO & 39.2$\pm$0.55 & \underline{30.3}$\pm$0.58 & \underline{82.0}$\pm$0.46 & \textbf{30.3}$\pm$0.30 & 44.6$\pm$0.31 & 45.3\\
 
 SAPO & \underline{40.3}$\pm$0.12 & 28.4$\pm$0.65 & 81.9$\pm$0.45 & \underline{30.1}$\pm$0.12 & 44.8$\pm$0.10 & 45.1\\
 80/20-Rule & 34.3$\pm$0.65 & 25.9$\pm$0.12 & 77.2$\pm$0.94 & 27.6$\pm$0.17 & 45.1$\pm$0.26 & 42.0\\
 
 BAPO & 35.6$\pm$0.42 & 24.7$\pm$0.34 & 77.8$\pm$0.52 & 28.3$\pm$0.12 & 42.3$\pm$0.29 & 41.7\\
 \midrule
 \rowcolor{tableColor} \textbf{HTPO (ours)} & \textbf{41.5}$\pm$0.50 & \textbf{30.4}$\pm$0.14 & \textbf{83.0}$\pm$0.54 &  \textbf{30.3}$\pm$0.05 & \textbf{47.6}$\pm$0.23 & \textbf{46.6}\\
 \rowcolor{tableColor} \textbf{APG (vs. DAPO)} & +8.6 & +6.7 & +7.1 & +0.9 & +1.9 &  +5.1\\
 \rowcolor{tableColor} \textbf{RPG (vs. DAPO)} & +26.1\% & +28.3\% & +9.4\% & +3.1\% & +4.2\% & +12.3\%\\
 \midrule
\multicolumn{7}{c}{\textbf{\emph{Qwen3-8B-Instruct}}} \\
\midrule
 GRPO$^\dagger$  & 71.3$\pm$0.94 & 58.6$\pm$0.55 & 92.2$\pm$0.52 & 34.8$\pm$0.05 & 53.5$\pm$0.19 & 62.1\\
 DAPO & 73.2$\pm$0.63 & \underline{64.7}$\pm$0.77 & 93.8$\pm$1.06 & 33.8$\pm$0.09 & 59.3$\pm$0.35 & 65.0\\
 GSPO & 72.9$\pm$0.71 & 63.7$\pm$0.49 & \underline{94.1}$\pm$0.24 & \underline{35.9}$\pm$0.22 & 58.4$\pm$0.17 & 65.0\\
 SAPO & \underline{74.0}$\pm$0.67 & 63.6$\pm$0.45 & 94.0$\pm$0.12 & 33.1$\pm$0.08 & 59.9$\pm$0.21 & 64.9\\
 80/20-Rule & 72.9$\pm$0.22 & 63.4$\pm$0.68 & 93.8$\pm$0.52 & 34.3$\pm$0.08 & \underline{60.0}$\pm$0.09 & 64.9\\
 BAPO & 72.8$\pm$0.75 & 59.2$\pm$0.80 & 93.3$\pm$0.22 & \textbf{36.5}$\pm$0.17 & 56.9$\pm$0.12 & 63.7 \\
 \midrule
\rowcolor{tableColor} \textbf{HTPO (ours)} & \textbf{75.2}$\pm$0.53 & \textbf{66.0}$\pm$0.45 & \textbf{96.3}$\pm$0.33 & \textbf{36.5}$\pm$0.12 & \textbf{60.9}$\pm$0.15 & \textbf{67.0} \\
 \rowcolor{tableColor} \textbf{APG (vs. DAPO)} & +2.0 & +1.3 & +2.5 & +2.7 & +1.6 &  +2.0\\
 \rowcolor{tableColor} \textbf{RPG (vs. DAPO)} & +2.7\% & +2.0\% & +2.7\% & +8.0\% & +2.7\% & +3.1\%\\
\bottomrule[0.5mm]
\end{tabular}}
\end{table*}

\textbf{Comparative performance analysis.} As presented in Tab.\ref{tab:main_results}, our HTPO achieves superior results and consistently outperforms these competing RL methods, most notably surpassing both the GRPO$^\dagger$ and DAPO baselines substantially across all 5 benchmarks. The efficacy of our method is particularly pronounced when applied to pre-trained base models. For instance, the absolute performance gain (APG) over DAPO on the AIME 2024 and AIME 2025 benchmarks reaches a remarkable +8.6 and +6.7, respectively. This suggests that our token-level divide-and-conquer strategy can bring more substantial improvements for models with weaker reasoning abilities (\emph{i.e.}, base models have greater untapped reasoning potential as they have not been heavily optimized via instruction-tuning). For the Qwen3-8B-Instruct model, although the margins of improvement are slightly narrower due to stronger baseline performance, HTPO still achieves consistent gains, proving its robustness and effectiveness across the two typical RL application scenarios. 


\textbf{Training Dynamics.} Analyzing the training process can help to reveal the mechanism behind HTPO’s performance gains. In Fig.\ref{fig:training_dynamics}, we illustrate the training dynamics of the Qwen3-8B-Base model regarding entropy, validation accuracy on AIME'24, and response length. It can be found that GRPO suffers from entropy collapse, while the entropy is significantly increased in DAPO. However, in DAPO, the entropy merely increases blindly without effectively translating into improved policy performance. Both the over-exploitation in GRPO and over-exploration in DAPO lead to suboptimal policy improvements. In contrast, our HTPO neither suffers from entropy collapse nor indiscriminately increases entropy. Its balanced entropy profile suggests a well-controlled exploration-exploitation trade-off and translates to consistent improvements: HTPO achieves steady growth in validation scores, ultimately outperforming all baselines. In addition, compared to GRPO, the model trained by HTPO is more efficient (the response length is shorter, but the accuracy is significantly higher). These results demonstrate that HTPO provides more reliable and effective optimization compared to standard GRPO-style methods that treat all tokens uniformly. In Appendix \ref{sec:additional_training_dynamics}, we provide more results about training dynamics.

\begin{figure*}[ht]
    \centering
    \includegraphics[width=1.0\linewidth]{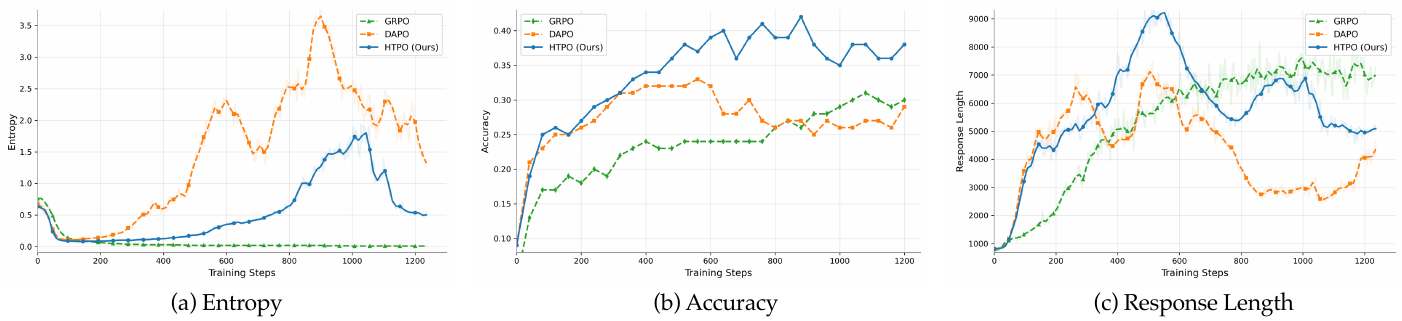}
    \caption{Training dynamics of the Qwen3-8B-Base model. The plots show the evolution of (a) policy entropy, (b) evaluation accuracy, and (c) response length. HTPO exhibits a well-controlled exploration-exploitation trade-off, characterized by no entropy collapse and better evaluation accuracy.}
    \label{fig:training_dynamics}
\end{figure*}

\subsection{Further Analysis}

\textbf{Scaling Test-Time Computation.} We next investigate how models trained with HTPO perform when allocated more computational resources at test time. As shown in Fig.\ref{fig:test_time_scale}, the HTPO-trained model (based on Qwen3-8B-Base) maintains a consistent and significant performance advantage
over the DAPO baseline on AIME25. Notably,
the performance gap widens as the number of samples increases, growing from a +5.5 point lead at pass@16 to a +11.4 point lead at pass@128. This outcome validates the adaptive strategy in our HTPO effectively fosters robust exploration (\emph{i.e.}, evaluated by large Pass@k) without sacrificing precision (the model maintains accuracy at Pass@1).

\begin{figure}[htbp]
    \centering
    \begin{minipage}{0.48\textwidth}
        \centering
        \includegraphics[width=\linewidth]{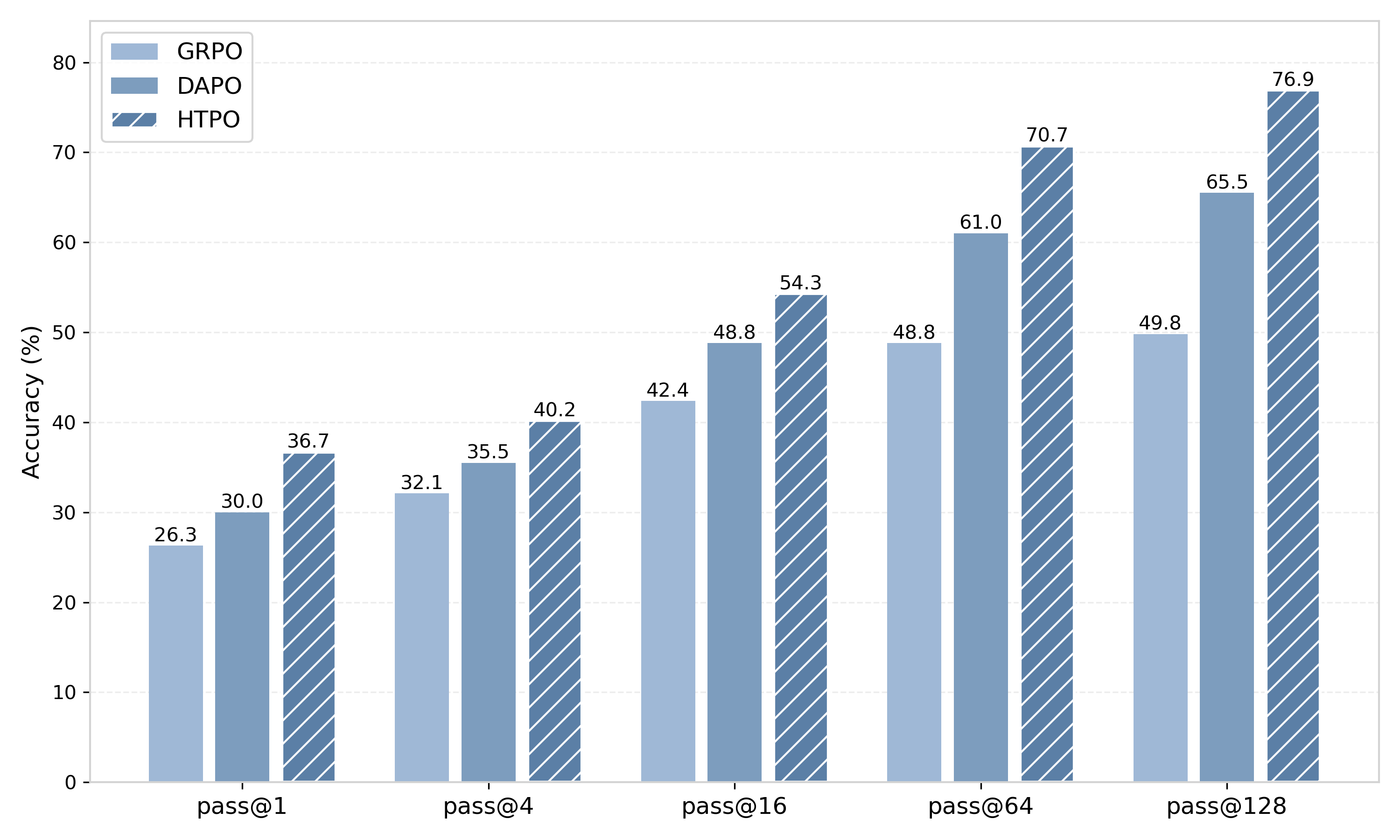}
        \captionof{figure}{Impact of scaling test-time compute on the AIME'25 benchmark (Pass@k). More results are provided in Fig.\ref{fig:test_time_scale_sup} in Appendix \ref{sec:additional_test_time_scale}.}
        \label{fig:test_time_scale}
    \end{minipage}%
    \hfill 
    \begin{minipage}{0.48\textwidth}
        \centering
         \resizebox{\linewidth}{!}{\begin{tabular}{ccccccc}
    \toprule[0.5mm]
     \multirow{2}{*}{\textbf{Method}} & \multicolumn{3}{c}{\textbf{LiveCodeBench}} & \textbf{ALFWorld} & \multicolumn{2}{c}{\textbf{WebShop}} \\
    \cmidrule(lr){2-4} \cmidrule(lr){6-7}
    & Pass@1 & Pass@5 & Pass@10 & Succ. & Score & Succ. \\
    \midrule
     GRPO$^\dagger$  & 24.2 & 30.2 & 31.9 & 59.3 & 55.4 & 31.4 \\
     DAPO & 26.8 & 32.8 & 34.3 & 60.0 & 57.3 & 31.8\\
     \midrule
     \rowcolor{tableColor} \textbf{HTPO (ours)} & 28.7 & 34.7 & 36.1 & 61.4 & 58.7 & 33.8 \\
     \rowcolor{tableColor} \textbf{APG (vs. DAPO)} & +1.9 & +1.9 & +1.8 & +1.4 & +1.4 & +2.0\\
     \rowcolor{tableColor} \textbf{RPG (vs. DAPO)} & +7.1\% & +5.8\% & +5.2\% & +2.3\% & +2.4\% & +6.3\%\\
    \bottomrule[0.5mm]
    \end{tabular}}
        \captionof{table}{Performance comparison on code and agentic reasoning benchmarks. ``Succ.'' denotes success rate. For WebShop, we report both the task score and success rate (\%). All results are based on Qwen3-8B. Qwen3-8B-Base is excluded from this comparison due to its weak instruction-following capabilities, which prevent it from working properly on these tasks.} 
        \label{tab:other_datasets}
    \end{minipage}
\end{figure}

\textbf{Generalization ability to other domains.}  
As outlined in Sec.\ref{sec:experimental_setup}, we evaluated our method primarily on mathematical reasoning benchmarks. Here, we further evaluate whether HTPO can still surpass DAPO on out-of-distribution test sets, including code tasks (LiveCodeBench \cite{livecodebench}) and agentic tasks (ALFWorld \cite{ALFworld}, WebShop \cite{WebShop}). The results are illustrated in Tab.\ref{tab:other_datasets}, using the same setup described in Sec.\ref{sec:experimental_setup}. From these results, we observe that HTPO still effectively outperforms vanilla DAPO on the out-of-distribution test datasets. This finding suggests that high-entropy tokens may be associated with the generalization capabilities of reasoning models. Encouraging exploration on hard problems and amplifying the gradient contributions of high-entropy tokens could potentially enhance the generalization ability of reasoning models.

\subsection{Ablation Studies}

\textbf{The contribution of the hierarchical groups to performance improvement.} In our HTPO algorithm, we hierarchically divide response tokens into eight functional groups and design tailored optimization objectives for each group. Therefore, a key question about our method is the effectiveness of these group-specific objectives. We conduct relevant ablation experiments and report the results in Tab.\ref{tab:method_ablation}. For each group, we apply the optimization objectives designed in Sec.\ref{sec:method}, while retaining the original optimization objective (Eq.(\ref{eq:base_objective})) in the other groups. The results confirm that tailoring the optimization objectives according to the specific roles of each token group yields superior results compared to the original objective. Furthermore, these distinct groups exhibit complementary characteristics. When integrated into HTPO, we can achieve the best performance. In Appendix \ref{sec:group_policy_entropy}, we further illustrate the entropy evolution patterns of different groups in Fig.\ref{fig:group_entropy_ablation}, which can intuitively demonstrate the specific effects of these group-specific optimization objectives.

\begin{wraptable}{r}{0.35\textwidth}
\centering
\caption{Method ablation studies. For \textbf{Group 3/5/7}, the optimization objectives are the same as the original (Eq.(\ref{eq:base_objective})), so the corresponding experiments are missing.}
\label{tab:method_ablation}
\resizebox{\linewidth}{!} {
\begin{tabular}{c|ccc}
\toprule[0.5mm]
  \textbf{Groups} & \textbf{AIME'24} & \textbf{AIME'25} & \textbf{Avg.}\\
\midrule

   Baseline & 33.3 & 23.8 &  28.6 \\
 
   Group 1  & 35.4 & 25.7 &  30.6 \\

   Group 2  & 35.7 & 25.1 &  30.4 \\

   Group 4  & 36.7 & 28.5 &  32.6 \\
   
   Group 6  & 35.8 & 26.0 & 30.9 \\

   Group 8  & 36.4 & 25.4 &  30.9 \\

   HTPO (ours) & 41.7 & 30.2 & 36.0  \\

\bottomrule[0.5mm]
\end{tabular}}
\end{wraptable}
\textbf{Hyperparameter ablation studies.} In Tab.\ref{tab:hyperparameter_ablation}(a), we present ablation experiments on the clipping ratios used in the optimization objectives of \textbf{Group 1} and \textbf{Group 6}. Both results indicate that excessively high or low clipping ratios lead to performance degradation. This finding confirms that a conservative ratio cannot effectively achieve the goal of detaching some specific tokens, whereas an overly aggressive ratio risks compromising important learning signals. In Tab.\ref{tab:hyperparameter_ablation}(b), we show the ablation experiments on the difficulty threshold $\tau_{diff}$. A threshold of $\tau_{diff} = 0.0$ forces HTPO to always treat problems as ``hard'', thereby always encouraging exploration. Conversely, $\tau_{diff} = 1.0$ forces HTPO to always treat problems as ``easy'', thus always promoting exploitation. The results clearly show that fixed
strategies (pure exploitation or exploration) are both suboptimal. The best performance is achieved with intermediate thresholds, with $\tau_{diff} = 0.5$ achieving
the highest average accuracy of 36.0\%. This confirms that the power of HTPO lies not in forcing a single behavior, but in its ability to dynamically switch between exploration and exploitation based on task difficulty.

\begin{table}[ht]
\caption{Hyperparameter ablation studies. (a) Ablation experiments on the clipping ratios ($\rho_1$, $\rho_2$). (b) Ablation experiments on the difficulty threshold ($\tau_{diff}$).}
\label{tab:hyperparameter_ablation}
    \begin{subtable}{0.55\linewidth}
    \centering
    \resizebox{!}{1.2cm}{\begin{tabular}{c|ccc|c|ccc}
\toprule[0.5mm]
  \textbf{$\rho_1$} & \textbf{AIME'24} & \textbf{AIME'25} & \textbf{Avg.} & \textbf{$\rho_2$} & \textbf{AIME'24} & \textbf{AIME'25} & \textbf{Avg.} \\
\midrule

   0.002 & 33.4 & 25.2 & 29.3 & 0.006 & 32.2 & 24.1 & 28.2 \\
    
   0.006 & 35.4 & 25.7 & 30.6 & 0.01 & 35.7 & 25.5 & 30.6 \\

   0.01  & 32.9 & 24.6 & 28.8 & 0.015  & 34.9 & 25.3 & 30.1 \\
 
   0.02  & 33.8 & 23.4 & 28.6 & 0.02  & 35.8 & 26.0 & 30.9 \\

   0.05  & 33.2 & 25.1 & 29.2 & 0.03  & 33.2 & 24.9 & 29.1 \\

    0.1  & 33.0 & 24.7 & 28.9 & 0.05  & 34.6 & 25.2 & 29.9 \\
    
   0.2  & 32.8 & 23.3 & 28.1 & 0.1  & 33.9 & 24.8 & 29.4 \\

   0.3  & 31.6 & 24.1 & 27.9 & 0.2  & 33.1 & 23.9 & 28.5 \\
   
\bottomrule[0.5mm]
\end{tabular}}
\caption{}
    \end{subtable}
    \begin{subtable}{0.45\linewidth}
    \centering
    \resizebox{!}{1.2cm}{\begin{tabular}{c|ccc}
\toprule[0.5mm]
  \textbf{Treshold ($\tau_{diff}$)} & \textbf{AIME'24} & \textbf{AIME'25} & \textbf{Avg.}\\
\midrule
 
   0.0  & 38.8 & 28.5 &  33.7 \\

   0.3  & 39.9 & 28.5 &  34.2 \\

   0.4  & 40.5 & 28.2 &  34.4 \\

   0.5  & 41.7 & 30.2 &  36.0 \\

   0.6  & 40.1 & 28.0 &  34.1 \\

   0.7  & 39.6 & 28.1 &  33.9 \\

   1.0 & 39.2 & 27.9 &   33.6 \\

\bottomrule[0.5mm]
\end{tabular}}
\caption{}
    \end{subtable}
\end{table}

\section{Conclusion}

In this paper, motivated by the insight that different tokens play distinct roles during the reasoning process, we identify a critical limitation in conventional RL methods: the uniform treatment of all tokens under a shared optimization objective. To address this, we innovatively divide the response tokens into eight functional groups from three aspects (\emph{i.e.}, prompt difficulty, answer correctness, and token entropy) and design specialized optimization objectives to facilitate the effective execution of each token's expected functionality. Extensive experiments demonstrate that our HTPO algorithm substantially outperforms strong DAPO and GRPO baselines, confirming the efficacy of our token-level divide-and-conquer idea. Analogous to research on dense credit assignment, we hope this work can provide new inspiration for the LLM RL community: advocating for the design of optimal optimization objectives tailored for specific token roles. The limitations and social impacts of our work are discussed in Appendix \ref{sec:limitation_social_impacts}.

\section*{Acknowledgments}

We gratefully acknowledge Xiaohongshu Inc. for providing computing resources that greatly supported the completion of this work. This work was also supported in part by the National Natural
Science Fund of China (No.62371295), the Shanghai
Jiao Tong University AI for Engineering Initiative
(No.WH410263001/001), and the Science and
Technology Commission of Shanghai Municipality
(No.22DZ2229005).

\bibliographystyle{plain} 
\bibliography{ref}

@article{REINFORCE,
author = {Ronald J. Williams},
title = {Simple Statistical Gradient-Following Algorithms for Connectionist Reinforcement Learning},
journal = {Mach. Learn},
year = 1992}

@article{PPO,
author = {John Schulman and Filip Wolski and Prafulla Dhariwal and Alec Radford and Oleg Klimov},
title = {Proximal policy optimization algorithms},
journal = {arXiv preprint arXiv: 1707.06347},
year = 2017}

@article{GRPO,
author = {Zhihong Shao and Peiyi Wang and Qihao Zhu and Runxin Xu and Junxiao Song and Xiao Bi and Haowei Zhang and Mingchuan Zhang and Y.K. Li and Y. Wu and Daya Guo},
title = {DeepSeekMath: Pushing the Limits of Mathematical
Reasoning in Open Language Modelss},
journal = {arXiv preprint arXiv: 2402.03300},
year = 2024}

@article{Entropy-Mechanism,
author = {Ganqu Cui and Yuchen Zhang and Jiacheng Chen and Lifan Yuan and Zhi Wang and Yuxin Zuo and Haozhan Li and Yuchen Fan and Huayu Chen and Weize Chen and Zhiyuan Liu and Hao Peng and Lei Bai and Wanli Ouyang and Yu Cheng and Bowen Zhou and Ning Ding},
title = {The Entropy Mechanism of Reinforcement Learning for Reasoning Language Models},
journal = {arXiv preprint arXiv: 2505.22617},
year = 2025}

@article{DACE,
author = {Ang Li and Zhihang Yuan and Yang Zhang and Shouda Liu and Yisen Wang},
title = {Know When to Explore: Difficulty-Aware Certainty as a Guide for LLM Reinforcement Learning},
journal = {arXiv preprint arXiv: 2509.00125},
year = 2025}

@article{80/20-rule,
author = {Shenzhi Wang and Le Yu and Chang Gao and Chujie Zheng and Shixuan Liu and Rui Lu and
Kai Dang and Xionghui Chen and Jianxin Yang and Zhenru Zhang and Yuqiong Liu and An Yang and Andrew Zhao and Yang Yue and Shiji Song and Bowen Yu and Gao Huang and Junyang Lin},
title = {Beyond the 80/20 Rule: High-Entropy Minority Tokens Drive Effective Reinforcement Learning for LLM Reasoning},
journal = {arXiv preprint arXiv: 2506.01939},
year = 2025}

@article{DAPO,
author = {ByteDance Seed},
title = {DAPO: An Open-Source LLM Reinforcement Learning System at Scale},
journal = {arXiv preprint arXiv: 2503.14476},
year = 2025}

@article{VAPO,
author = {Yu Yue and Yufeng Yuan and Qiying Yu and Xiaochen Zuo and Ruofei Zhu and Wenyuan Xu and Jiaze Chen and Chengyi Wang and TianTian Fan and Zhengyin Du and Xiangpeng Wei and Xiangyu Yu and Gaohong Liu and Juncai Liu and Lingjun Liu and Haibin Lin and Zhiqi Lin and Bole Ma and Chi Zhang and Mofan Zhang and Wang Zhang and Hang Zhu and Ru Zhang and Xin Liu and Mingxuan Wang and YonghuiWu and Lin Yan},
title = {Vapo: Efficient and reliable reinforcement learning for advanced reasoning tasks},
journal = {arXiv preprint arXiv: 2504.05118},
year = 2025}

@article{GSPO,
author = {Chujie Zheng and Shixuan Liu and Mingze Li and Xiong-Hui Chen and Bowen Yu and Chang Gao and Kai Dang and Yuqiong Liu and Rui Men and An Yang and Jingren Zhou and Junyang Lin},
title = {Group Sequence Policy Optimization},
journal = {arXiv preprint arXiv: 2507.18071},
year = 2025}

@article{SAPO,
author = {Chang Gao and Chujie Zheng and Xiong-Hui Chen and Kai Dang and Shixuan Liu and Bowen Yu and An Yang and Shuai Bai and Jingren Zhou and Junyang Lin},
title = {Soft Adaptive Policy Optimization},
journal = {arXiv preprint arXiv: 2511.20347},
year = 2025}

@article{BAPO,
author = {Zhiheng Xi and Xin Guo and Yang Nan and Enyu Zhou and Junrui Shen and Wenxiang Chen and Jiaqi Liu and Jixuan Huang and Zhihao Zhang and Honglin Guo and Xun Deng and Zhikai Lei and Miao Zheng and Guoteng Wang and Shuo Zhang and Peng Sun and Rui Zheng and Hang Yan and Tao Gui and Qi Zhang and Xuanjing Huang},
title = {BAPO: Stabilizing Off-Policy Reinforcement Learning for LLMs via Balanced Policy Optimization with Adaptive Clipping},
journal = {arXiv preprint arXiv: 2510.18927},
year = 2025}

@article{DoesRL,
author = {Yang Yue and Zhiqi Chen and Rui Lu and Andrew Zhao and Zhaokai Wang and Yang Yue and Shiji Song and Gao Huang},
title = {Does reinforcement learning really incentivize reasoning capacity in llms beyond the base model?},
journal = {arXiv preprint arXiv: 2504.13837},
year = 2025}

@article{RLone,
author = {YipingWang and Qing Yang and Zhiyuan Zeng and Liliang Ren and Lucas Liu and Baolin Peng and Hao Cheng and Xuehai He and Kuan Wang and Jianfeng Gao and Weizhu Chen and Shuohang Wang and Simon Shaolei Du and Yelong Shen},
title = {Reinforcement learning for reasoning in large language models with one training example},
journal = {arXiv preprint arXiv: 2504.20571},
year = 2025}

@article{InstructGPT,
author = {Long Ouyang and Jeffrey Wu and Xu Jiang and Diogo Almeida and Carroll Wainwright and Pamela Mishkin and Chong Zhang and Sandhini Agarwal and Katarina Slama and Alex Ray and et al},
title = {Training language models to follow instructions with human feedback},
journal = {In NeurIPS},
year = 2022}

@article{RLAIF,
author = {Shangmin Guo and Biao Zhang and Tianlin Liu and Tianqi Liu and Misha Khalman and Felipe Llinares and Alexandre Rame and Thomas Mesnard and Yao Zhao and Bilal Piot and Johan Ferret and Mathieu Blondel},
title = {Direct Language Model Alignment from Online AI Feedback},
journal = {arXiv preprint arXiv: 2402.04792},
year = 2024}

@article{OpenAI-o1,
author = {OpenAI},
title = {Learning to reason with LLMs},
journal = {https://openai.com/index/
learning-to-reason-with-llms/},
year = 2024}

@article{gpt5.2,
author = {OpenAI},
title = {Introducing GPT-5.2},
journal = {https://openai.com/index/introducing-gpt-5-2/},
year = 2025}

@article{gemini-3,
author = {Gemini Team},
title = {Gemini 3 Flash: frontier intelligence built for speed},
journal = {https://blog.google/products-and-platforms/products/gemini/gemini-3-flash//},
year = 2025}

@article{DeepSeek-R1,
author = {DeepSeek-AI},
title = {Deepseek-r1: Incentivizing reasoning capability in llms via reinforcement learning},
journal = {arXiv preprint arXiv: 2501.12948},
year = 2025}

@article{Claude-3.7,
author = {Anthropic},
title = {Claude 3.7 Sonnet},
journal = {https://www.anthropic.com/claude/sonnet},
year = 2025}

@article{Claude-4.6,
author = {Anthropic},
title = {Introducing Claude Opus 4.6},
journal = {https://www.anthropic.com/news/claude-opus-4-6},
year = 2025}

@article{Qwen3,
author = {Team Qwen},
title = {Qwen3 technical report},
journal = {arXiv preprint arXiv: 2505.09388},
year = 2025}

@article{Kimi-1.5,
author = {Kimi Team},
title = {Kimi k1.5: Scaling reinforcement learning with llms},
journal = {arXiv preprint arXiv: 2501.12599},
year = 2025}

@article{Kimi-k2,
author = {Kimi Team},
title = {Kimi K2: Open Agentic Intelligence},
journal = {arXiv preprint arXiv: 2507.20534},
year = 2025}

@article{GLM-4.5,
author = {GLM-4.5 Team},
title = {GLM-4.5: Agentic, Reasoning, and Coding (ARC) Foundation Models},
journal = {arXiv preprint arXiv: 2508.06471},
year = 2025}

@article{Tongyi-DeepSearch,
author = {Tongyi DeepResearch Team},
title = {Tongyi DeepResearch Technical Report},
journal = {arXiv preprint arXiv: 2510.24701},
year = 2025}

@article{verl,
author = {Guangming Sheng and Chi Zhang and Zilingfeng Ye and Xibin Wu and Wang Zhang and Ru Zhang and Yanghua Peng and Haibin Lin and Chuan Wu},
title = {Hybridflow: A flexible and efficient rlhf framework},
journal = {arXiv preprint arXiv: 2409.19256},
year = 2024}

@article{DPO,
author = {Rafael Rafailov and Archit Sharma and Eric Mitchell and Stefano Ermon and Christopher D. Manning and Chelsea Finn},
title = {Direct Preference Optimization: Your Language Model is Secretly a Reward Model},
journal = {In NeurIPS},
year = 2023}

@article{SimPO,
author = {Yu Meng and Mengzhou Xia and Danqi Chen},
title = {SimPO: Simple Preference Optimization with a Reference-Free Reward},
journal = {In NeurIPS},
year = 2024}

@article{KTO,
author = {Kawin Ethayarajh and Winnie Xu and Niklas Muennighoff and Dan Jurafsky and Douwe Kiela},
title = {KTO: Model Alignment as Prospect Theoretic Optimization},
journal = {In ICML},
year = 2024}

@article{ORPO,
author = {Jiwoo Hong and Noah Lee and James Thornen},
title = {ORPO: Monolithic Preference Optimization without Reference Model},
journal = {arXiv preprint arXiv: 2403.07691},
year = 2024}

@article{BCO,
author = {Seungjae Jung and Gunsoo Han and Daniel Wontae Nam and Kyoung-Woon On},
title = {Binary Classifier Optimization for Large Language Model Alignment},
journal = {In ACL},
year = 2025}

@article{simplerl-zoo,
author = {Weihao Zeng and Yuzhen Huang and Qian Liu and Wei Liu and Keqing He and Zejun Ma and Junxian He},
title = {Simplerl-zoo: Investigating and taming zero reinforcement learning for open base models in the wild},
journal = {arXiv preprint arXiv: 2503.18892},
year = 2025}

@article{open-reasoner,
author = {Jingcheng Hu and Yinmin Zhang and Qi Han and Daxin Jiang and Xiangyu Zhang and Heung-Yeung Shum},
title = {Open-reasoner-zero: An open source approach to scaling up reinforcement learning on the base model},
journal = {arXiv preprint arXiv: 2503.24290},
year = 2025}

@article{critical-tokens1,
author = {Jean Vassoyan and Nathanaël Beau and Roman Plaud},
title = {Ignore the kl penalty! boosting exploration on critical tokens to enhance rl fine-tuning},
journal = {arXiv preprint arXiv: 2502.06533},
year = 2025}

@article{critical-tokens2,
author = {Zicheng Lin and Tian Liang and Jiahao Xu and Xing Wang and Ruilin Luo and Chufan Shi and Siheng Li and Yujiu Yang and Zhaopeng Tu},
title = {Critical tokens matter: Token-level contrastive estimation enhence llm’s reasoning capability},
journal = {arXiv preprint arXiv: 2411.19943},
year = 2024}

@article{aime24,
author = {Mathematical Association of America},
title = {American Invitational Mathematics Examination 2024},
journal = {https://artofproblemsolving.com/wiki/index.php/AIME\_Problems\_and\_Solutions},
year = 2024}

@article{aime25,
author = {Mathematical Association of America},
title = {American Invitational Mathematics Examination 2025},
journal = {https://artofproblemsolving.com/wiki/index.php/AIME\_Problems\_and\_Solutions},
year = 2025}

@article{amc23,
author = {Mathematical Association of America},
title = {2023 AMC},
journal = {https://artofproblemsolving.com/wiki/index.php/2023\_AMC\_8},
year = 2023}

@article{olympiadbench,
author = {Chaoqun He and Renjie Luo and Yuzhuo Bai and Shengding Hu and Zhen Thai and Junhao Shen and Jinyi Hu and Xu Han and Yujie Huang and Yuxiang Zhang and Jie Liu and Lei Qi and Zhiyuan Liu and Maosong Sun},
title = {OlympiadBench: A challenging benchmark for promoting AGI with olympiad-level bilingual multimodal scientific problems},
journal = {In Proceedings of the 62nd Annual
Meeting of the Association for Computational Linguistics},
year = 2024}

@article{livecodebench,
author = {Naman Jain and King Han and Alex Gu and Wen-Ding Li and Fanjia Yan and Tianjun Zhang and Sida Wang and Armando Solar-Lezama and Koushik Sen and Ion Stoica},
title = {Livecodebench: Holistic and contamination free evaluation of large language models for code},
journal = {arXiv preprint arXiv: 2403.07974},
year = 2024}

@article{ALFworld,
author = {Mohit Shridhar and Xingdi Yuan and Marc-Alexandre Cote and Yonatan Bisk and Adam Trischler and Matthew Hausknecht},
title = {ALFWorld: Aligning text and embodied environments for interactive learning},
journal = {In Proceedings of the International Conference on Learning Representations},
year = 2021}

@article{WebShop,
author = {Shunyu Yao and Howard Chen and John Yang and Karthik Narasimhan},
title = {WebShop: Towards scalable realworld
web interaction with grounded language agents},
journal = {In Proceedings of the Advances in Neural
Information Processing Systems},
year = 2022}

@article{theory1,
author = {Robin Young},
title = {Why Is RLHF Alignment Shallow? A Gradient Analysis},
journal = {arXiv preprint arXiv: 2603.04851},
year = 2026}


\clearpage
\appendix
\section*{Appendix}

\section{Limitations and Social Impacts}
\label{sec:limitation_social_impacts}

\subsection{Limitations}
\label{sec:limitation}

Previous RL improvement methods have primarily focused on dense credit assignment (\emph{i.e.}, how to allocate appropriate token-level rewards rather than distributing the outcome reward across the entire sequence). Our work introduces a new, orthogonal direction for improving RL algorithms: designing optimal optimization objectives tailored to the specific roles of different tokens, rather than applying a shared objective. Therefore, combining our method with dense credit assignment methods to assign both adaptive rewards and optimization objectives to each token is a promising avenue for future work. 

Furthermore, while extensive experimental results demonstrate the effectiveness of our method, we do not claim that the optimization objectives we designed are optimal within each group, nor that our token partitioning strategy is optimal. Future work should also focus on continuously designing better optimization objectives to maximize the utility of each token. Finally, the training and validation are mainly based on mathematical reasoning datasets. Despite the encouraging results on code and agentic tasks in Tab.\ref{tab:other_datasets}, extending our method from the reasoning RL scenarios to broader agentic RL scenarios to address inherent challenges in agent training is also a key direction for future work.

\subsection{Social Impacts and Ethics}
\label{sec:social_impacts}

Our work proposes a novel reinforcement learning algorithm that addresses the limitation of treating all tokens uniformly in previous GRPO-style methods. Our algorithm substantially outperforms strong DAPO and GRPO baselines and can effectively enhance the model's capabilities in complex reasoning tasks. As an RL algorithm, it does not inherently raise particular ethical concerns or negative social impacts. All datasets used are publicly available, and all qualitative visualizations are based on experimental metrics, ensuring no infringement on personal privacy. 

However, given that RL algorithms will be utilized to train LLMs, potential risks may arise in downstream applications. Specifically, the enhanced reasoning capabilities could be exploited for generating more convincing yet biased content, sophisticated malware, or logically consistent disinformation, posing significant security challenges. Consequently, careful curation of training datasets and explicit efforts to mitigate potential biases remain essential. Overall, the responsible and transparent deployment of our RL algorithm is crucial to ensure positive societal outcomes.

\section{The Overall Algorithmic Procedure of HTPO}
\label{sec:whole_algorithm}

We present the overall algorithmic procedure of our HTPO in Alg.\ref{algorithm1}. Since our method only modifies the RL optimization objective without introducing any additional models or computations during the rollout and forward stages, the extra computational cost is minimal. Specifically, our method only requires an additional difficulty estimate for each prompt and an entropy value for each token. The former is straightforward to calculate once the correctness of each response is determined, while the latter can be derived from the full probabilities over the vocabulary prior to token sampling during rollout. It is also cheap to calculate entropy based on the existing full probabilities. Therefore, the overall computational cost of our algorithm is basically the same as that of DAPO \cite{DAPO}. 

\begin{algorithm}
	\renewcommand{\algorithmicrequire}{\textbf{Input:}}
	\renewcommand{\algorithmicensure}{\textbf{Output:}}
	\caption{HTPO: Hierarchical Token-level Objective Control Policy Optimization}
	\label{algorithm1}
	\begin{algorithmic}[1]
     \REQUIRE initial policy model $\pi_{\theta_{init}}$; prompt dataset $\mathcal{D}$; hyperparameters: $\rho_1$, $\rho_2$, $\tau_{diff}$
		\STATE policy model $\pi_\theta \leftarrow \pi_{\theta_{init}}$
      \FOR {iteration = $1, \dots, T$} 
      \FOR {step = $1, \dots, M$} 
      \STATE Sample a batch of questions $\mathcal{Q}$ from $\mathcal{D}$ 
      \STATE Update the old policy model $\pi_{\theta_{old}} \leftarrow \pi_\theta$
      \STATE Sample $G$ responses $\{o_i\}_{i=1}^G \sim \pi_{\theta_{old}}(\cdot|q)$ for each question $q \in \mathcal{Q}$
      \STATE Compute rewards $\{r_i\}_{i=1}^G$ for each sampled response $o_i$
      \STATE Compute difficulty score ${\rm diff}(q;\pi_{\theta_{old}})$ for each question $q$, and divide $\mathcal{Q}$ into hard batch $\mathcal{Q}_h = \{q_i|{\rm diff}(q_i;\pi_{\theta_{old}}) > \tau_{diff}\}$ and easy batch $\mathcal{Q}_e = \{q_j|{\rm diff}(q_j;\pi_{\theta_{old}}) \leq \tau_{diff}\}$
      \FOR {$q \in \mathcal{Q}$}
      \FOR {response = $o_1, \dots, o_G$ of $q$}
      \STATE Calculate entropy for each token $o_{i,t}$
       \STATE Divide each token into a low-entropy or a high-entropy token according to the 80/20 rule
       \STATE \textbf{Group 1} = \{$q \in \mathcal{Q}_h$, ${\rm verify}(o_i)=1$, $o_{i,t}$ is low-entropy\}
       \STATE \textbf{Group 2} = \{$q \in \mathcal{Q}_h$, ${\rm verify}(o_i)=1$, $o_{i,t}$ is high-entropy\}
       \STATE \textbf{Group 3} = \{$q \in \mathcal{Q}_h$, ${\rm verify}(o_i)=0$, $o_{i,t}$ is low-entropy\}
       \STATE \textbf{Group 4} = \{$q \in \mathcal{Q}_h$, ${\rm verify}(o_i)=0$, $o_{i,t}$ is high-entropy\}
       \STATE \textbf{Group 5} = \{$q \in \mathcal{Q}_e$, ${\rm verify}(o_i)=1$, $o_{i,t}$ is low-entropy\}
       \STATE \textbf{Group 6} = \{$q \in \mathcal{Q}_e$, ${\rm verify}(o_i)=1$, $o_{i,t}$ is high-entropy\}
       \STATE \textbf{Group 7} = \{$q \in \mathcal{Q}_e$, ${\rm verify}(o_i)=0$, $o_{i,t}$ is low-entropy\}
       \STATE \textbf{Group 8} = \{$q \in \mathcal{Q}_e$, ${\rm verify}(o_i)=0$, $o_{i,t}$ is high-entropy\}
       
      \IF{$o_{i,t}\in$ \textbf{Group 1}}
       \STATE $J_{i,t}(\theta) = \begin{cases}
    {\rm min}\Big(r_t(\theta)\hat{A}_t,{\rm clip}(r_t(\theta), 1 - \epsilon, 1 + \epsilon)\hat{A}_t\Big), & t \notin \mathbb{I}[H_t < \tau_{\rho_1}^\mathcal{B}] \\
    0, & t \in \mathbb{I}[H_t < \tau_{\rho_1}^\mathcal{B}]\end{cases}$
    
      \ELSIF{$o_{i,t}\in$ \textbf{Group 2}}
      \STATE $J_{i,t}(\theta) = \begin{cases}(1+\epsilon) \cdot \frac{\pi_\theta(o_t|q,o_{<t})}{SG[\pi_\theta(o_t|q,o_{<t})]} \cdot \hat{A}_t,  &{\rm if} \; r_t(\theta) > 1 + \epsilon
     \\
    \frac{\pi_{\theta_{old}}(o_t|q,o_{<t})}{SG[\pi_\theta(o_t|q,o_{<t})]}\cdot \frac{\pi_\theta(o_t|q,o_{<t})}{SG[\pi_\theta(o_t|q,o_{<t})]}  \cdot \hat{A}_t, &{\rm if} \; r_t(\theta) < 1 - \epsilon\end{cases}$ 

      \ELSIF{$o_{i,t}\in$ \textbf{Group 4/8}}
      \STATE $J_{i,t}(\theta) = (1-\epsilon) \cdot \frac{\pi_\theta(o_t|q,o_{<t})}{SG[\pi_\theta(o_t|q,o_{<t})]} \cdot \hat{A}_t, \quad {\rm if} \; r_t(\theta) < 1 - \epsilon$

      \ELSIF{$o_{i,t}\in$ \textbf{Group 6}}
    \STATE $J_{i,t}(\theta) = \begin{cases}
    {\rm min}\Big(r_t(\theta)\hat{A}_t,{\rm clip}(r_t(\theta), 1 - \epsilon, 1 + \epsilon)\hat{A}_t\Big), & t \notin \mathbb{I}[H_t \geq \tau_{\rho_2}^\mathcal{B}] \\
    0, & t \in \mathbb{I}[H_t \geq \tau_{\rho_2}^\mathcal{B}]
    \end{cases}$
    
    \ELSE
      \STATE $J_{i,t}(\theta) = 
     {\rm min}\Big(r_t(\theta)\hat{A}_t,{\rm clip}(r_t(\theta), 1 - \epsilon, 1 + \epsilon)\hat{A}_t\Big)$
      \ENDIF
      \ENDFOR
      \ENDFOR

      \STATE Aggregate all $J_{i,t}(\theta)$ into $\mathcal{J}(\theta) = \mathbb{E}_{ \{o_i\}_{i=1}^G} \Big[\frac{1}{G|o_i|}\sum_{i=1}^{G}\sum_{t=1}^{|o_i|}J_{i,t}(\theta)\Big]$
      \STATE Update the policy model $\pi_\theta$ by maximizing the above $\mathcal{J}(\theta)$
      \ENDFOR
      \ENDFOR
		\ENSURE policy model $\pi_\theta$
	\end{algorithmic}  
\end{algorithm}

\section{Theoretical Analysis and Proofs}
\label{sec:theory}

In this section, we establish the theoretical foundation for HTPO by proving Theorem~\ref{thm:consistency} (Inter-Group Gradient Consistency), which states that the group-specific optimization objectives in HTPO do not generate conflicting training signals, and the produced gradients are mutually aligned. We begin by introducing the necessary assumptions, then derive the gradient formulations for each HTPO group, and finally present the complete proof.

\begin{assumption}[Bounded Gradient]\label{asmp:policy}
The policy $\pi_\theta$ satisfies:
\begin{enumerate}[leftmargin=2em, itemsep=0.2em, topsep=0.3em]
    \item The score function $\nabla_\theta \log \pi_\theta(o_t|s_t)$ is $L$-Lipschitz continuous in $\theta$, where we define $s_t := (q, o_{<t})$ for brevity.
    \item $\|\nabla_\theta \log \pi_\theta(o_t|s_t)\| \leq G_{\max}$ for all $o_t, s_t, \theta$.
\end{enumerate}
\end{assumption}

\begin{assumption}[Advantage-Weighted Gradient Direction Stability]\label{asmp:direction} There exists a unit reference direction $\boldsymbol{d}^*$ such that for all tokens $o_t$ in the training batch:
\begin{equation}
    \frac{\langle \hat{A}_t \cdot \nabla_\theta \log \pi_\theta(o_t | s_t),\; \boldsymbol{d}^* \rangle}{|\hat{A}_t| \cdot \|\nabla_\theta \log \pi_\theta(o_t | s_t)\|} \;\geq\; 1 - \eta.
    \label{eq:direction}
\end{equation}
where $\eta \in [0, 1)$ quantifies the maximum angular deviation of the advantage-weighted gradient direction from $\boldsymbol{d}^*$. This assumption is reasonable since if the advantage-weighted gradient direction is randomly scattered, RL training will not converge \cite{theory1}. The stable convergence of the model during actual training also indicates the macroscopic consistency of the per-token weighted gradient direction.
\end{assumption}

\begin{assumption}[Bounded Advantage]\label{asmp:advantage}
The advantage estimates satisfy $|\hat{A}_t| \leq A_{\max}$ for all $t$.
\end{assumption}

\begin{theorem}[Inter-Group Gradient Consistency]\label{thm:consistency}
Under Assumptions \ref{asmp:policy}--\ref{asmp:advantage}, for any two non-trivial groups $G_j, G_k$ (i.e., $\|\bm{g}^{(j)}\| > 0$ and $\|\bm{g}^{(k)}\| > 0$, where $\bm{g}^{(j/k)} = \E\bigl[\sum_{t \in G_{j/k}} \nabla_\theta J_t^{(j/k)}\bigr]$), the cosine similarity between their expected gradients satisfies:
\begin{equation}
    \operatorname{CosSim}\bigl(\boldsymbol{g}^{(j)}, \boldsymbol{g}^{(k)}\bigr) \;\geq\; \underbrace{\frac{(1-\epsilon)^2}{(1+\epsilon)^2}}_{\displaystyle \kappa(\epsilon)} \cdot (1-\eta)^2 \cdot \alpha_j \alpha_k \beta^2 \;-\; 2\eta.
    \label{eq:consistency}
\end{equation}
where $\eta$ is from Assumption \ref{asmp:direction}, $\beta = \mathbb{E}[\|\nabla_\theta \log \pi_\theta\|]/G_{\max} \in (0, 1]$, and $\alpha_{j/k} = \mathbb{E}[|\hat{A}_t| | o_t \in G_{j/k}]/A_{\max} \in (0, 1]$.

If based on strong advantage signals ($\alpha_j$ = $\alpha_k$ = $\beta$ = 1), for each $\epsilon$ value, when $\eta < \frac{\kappa(\epsilon) + 1 - \sqrt{2\kappa(\epsilon) + 1}}{\kappa(\epsilon)}$, the cosine similarity is strictly positive:
\begin{equation}
    \operatorname{CosSim}\bigl(\boldsymbol{g}^{(j)}, \boldsymbol{g}^{(k)}\bigr) \;>\; 0, \qquad \eta < \frac{\kappa(\epsilon) + 1 - \sqrt{2\kappa(\epsilon) + 1}}{\kappa(\epsilon)}.
\end{equation}
\end{theorem}

Before proceeding with the proof of Theorem \ref{thm:consistency}, we first derive the gradient formulations of the optimization objectives as follows:

\textbf{PPO Gradient Decomposition}. For the standard PPO-style clipped objective in Eq.(\ref{eq:base_objective}), the gradient can be decomposed as:
\begin{equation}
    \nabla_\theta J_{t}^{\text{clip}} = 
    \begin{cases}
        \hat{A}_t \cdot \nabla_\theta \log \pi_{\theta}(o_t|s_t), & \text{if } \; r_t(\theta) \in (1-\epsilon, 1+\epsilon), \\[0.4em]
        \bm{0}, & \text{if } \; r_t(\theta) > 1+\epsilon \text{ and } \hat{A}_t > 0, \\[0.4em]
        \bm{0}, & \text{if } \; r_t(\theta) < 1-\epsilon \text{ and } \hat{A}_t < 0.
    \end{cases}
    \label{eq:ppo-grad}
\end{equation}

\textbf{Gradient of HTPO Group $G_2$}. For tokens in group $G_2$, the gradient is:
\begin{equation}
    \nabla_\theta J_t^{(2)} = 
    \begin{cases}
        (1+\epsilon)\,\hat{A}_t \cdot \nabla_\theta \log \pi_{\theta}(o_t|s_t), & r_t(\theta) > 1+\epsilon, \\
        \dfrac{1}{r_t(\theta)}\,\hat{A}_t \cdot \nabla_\theta \log \pi_{\theta}(o_t|s_t), & r_t(\theta) < 1-\epsilon.
    \end{cases}
    \label{eq:g2-grad}
\end{equation}

\textbf{Gradient of HTPO Groups $G_4$ and $G_8$}. For tokens in groups $G_4$ and $G_8$, when $r_t(\theta) < 1-\epsilon$, the gradient takes the form:
\begin{equation}
    \nabla_\theta J_t^{(4)} = \nabla_\theta J_t^{(8)} = (1-\epsilon)\,\hat{A}_t \cdot \nabla_\theta \log \pi_{\theta}(o_t|s_t).
    \label{eq:g4-grad}
\end{equation}

\textbf{Unified Group Gradient Form}. The gradients of all non-trivial HTPO groups can be expressed in a unified form:
\begin{equation}
    \nabla_\theta J_t^{(k)} = w_t^{(k)} \cdot \hat{A}_t \cdot \nabla_\theta \log \pi_{\theta}(o_t|s_t).
    \label{eq:unified}
\end{equation}
where the scalar weight $w_t^{(k)}$ depends only on the group $G_k$ and the importance ratio $r_t(\theta)$:
\begin{equation}
    w_t^{(k)} =
    \begin{cases}
        0, & G_k \in \{G_1, G_6\} \text{ (detached tokens)}, \\[0.3em]
        1 + \epsilon, & G_k = G_2,\; r_t(\theta) > 1+\epsilon, \\[0.3em]
        \dfrac{1}{r_t(\theta)}, & G_k = G_2,\; r_t(\theta) < 1-\epsilon, \\[0.5em]
        0, & G_k \in \{G_3, G_5, G_7\},\; r_t(\theta) \notin (1-\epsilon, 1+\epsilon), \\[0.3em]
        1 - \epsilon, & G_k \in \{G_4, G_8\},\; r_t(\theta) < 1-\epsilon, \\
        1, & {\rm others}. \\[0.3em]
    \end{cases}
    \label{eq:weights}
\end{equation}

\textbf{Boundedness of Weights}. Based on Eq.(\ref{eq:weights}), all non-zero weights satisfy:
\begin{equation}
    1 - \epsilon \;\leq\; w_t^{(k)} \;\leq\; 1+\epsilon.
    \label{eq:weight-bound}
\end{equation}

We verify each case in Eq.(\ref{eq:weights}):
\begin{itemize}[leftmargin=2em, nosep]
    \item $G_2$ with $r_t > 1+\epsilon$: $w_t^{(2)} = 1 + \epsilon$.
    \item $G_2$ with $r_t < 1-\epsilon$: $w_t^{(2)} = 1/r_t$ and will be clipped up to $1 + \epsilon$.
    \item $G_3, G_5, G_7$ within clip: $w_t^{(k)} = 1$.
    \item $G_4, G_8$: $w_t^{(k)} = 1 - \epsilon$.
\end{itemize}
All non-zero weights take values in range [$1 - \epsilon$, $1 + \epsilon$].

Building on the preliminaries established above, we now proceed to prove Theorem \ref{thm:consistency} as follows:

\begin{proof} \medskip\noindent\textbf{Step 1: Advantage-weighted gradient decomposition relative to the reference direction.}

By Assumption \ref{asmp:direction}, for each token $o_t$, the projection of the advantage-weighted gradient onto the reference direction $\boldsymbol{d}^*$ satisfies:
\begin{equation}
    \frac{\langle \hat{A}_t \cdot \nabla_\theta \log \pi_\theta(o_t | s_t),\; \boldsymbol{d}^* \rangle}{|\hat{A}_t| \cdot \|\nabla_\theta \log \pi_\theta(o_t | s_t)\|} \geq 1 - \eta_t, \quad \text{where } \eta_t \in [0, \eta].
\end{equation}

This allows us to decompose the advantage-weighted gradient into a parallel component (along $\boldsymbol{d}^*$) and a perpendicular component:
\begin{align}
    \hat{A}_t \nabla_\theta \log \pi_\theta(o_t | s_t) = |\hat{A}_t| \cdot \|\nabla_\theta \log \pi_\theta(o_t | s_t)\| \cdot \Bigl[(1 - \eta_t)\,\boldsymbol{d}^* + \boldsymbol{e}_t\Bigr].
    \label{eq:decomp}
\end{align}

where $\boldsymbol{e}_t \perp \boldsymbol{d}^*$ and the vector $(1-\eta_t)\boldsymbol{d}^* + \boldsymbol{e}_t$ is a unit vector. Since $\boldsymbol{e}_t \perp \boldsymbol{d}^*$ and $\|\boldsymbol{d}^*\| = 1$, we have:
\begin{equation}
    \|(1-\eta_t)\boldsymbol{d}^* + \boldsymbol{e}_t\|^2 = (1-\eta_t)^2 + \|\boldsymbol{e}_t\|^2 = 1 \quad \Rightarrow \quad \|\boldsymbol{e}_t\| = \sqrt{2\eta_t - \eta_t^2}.
\end{equation}
Given $\eta_t \in [0, \eta]$ and the monotonicity of $f(x) = 2x - x^2$ on $[0, 1)$, we obtain the uniform bound:
\begin{equation}
    \|\boldsymbol{e}_t\| \leq \sqrt{2\eta - \eta^2} \leq \sqrt{2\eta}.
\end{equation}

\medskip\noindent\textbf{Step 2: Express group gradients in the reference frame.}

Using the unified gradient form from Eq.(\ref{eq:unified}), the expected gradient for group $G_k$ is:
\begin{equation}
    \bm{g}^{(k)} = \E\left[\sum_{t \in G_k} w_t^{(k)}\,\hat{A}_t \cdot \nabla_\theta \log \pi_\theta(o_t|s_t)\right].
    \label{eq:expected-grad}
\end{equation}

Substituting the decomposition from Eq.(\ref{eq:decomp}) into Eq.(\ref{eq:expected-grad}):

\begin{align}
    \boldsymbol{g}^{(k)} &= \mathbb{E}\left[\sum_{t \in G_k} w_t^{(k)} |\hat{A}_t| \|\nabla_\theta \log \pi_\theta(o_t | s_t)\| \cdot \bigl((1-\eta_t)\boldsymbol{d}^* + \boldsymbol{e}_t\bigr)\right] \nonumber \\
    &= \underbrace{\mathbb{E}\left[\sum_{t \in G_k} w_t^{(k)} |\hat{A}_t| \|\nabla_\theta \log \pi_\theta(o_t | s_t)\| (1-\eta_t)\right]}_{\displaystyle \coloneqq \rho_k} \boldsymbol{d}^* + \underbrace{\mathbb{E}\left[\sum_{t \in G_k} w_t^{(k)} |\hat{A}_t| \|\nabla_\theta \log \pi_\theta(o_t | s_t)\| \boldsymbol{e}_t\right]}_{\displaystyle \coloneqq \boldsymbol{\epsilon}^{(k)}}.
\end{align}

Thus, $\boldsymbol{g}^{(k)} = \rho_k \boldsymbol{d}^* + \boldsymbol{\epsilon}^{(k)}$, where $\boldsymbol{\epsilon}^{(k)} \perp \boldsymbol{d}^*$. A critical observation is that $\rho_k > 0$ for all non-trivial groups (both correct and wrong), because:
\begin{itemize}[leftmargin=2em, nosep]
    \item $w_t^{(k)} \geq 0$ (non-negative weights),
    \item $|\hat{A}_t| > 0$ (non-trivial groups have non-zero advantages),
    \item $\|\nabla_\theta \log \pi_\theta(o_t|s_t)\| > 0$ (non-zero score function),
    \item $(1-\eta_t) > 0$ (since $\eta < 1$).
\end{itemize}

\medskip\noindent\textbf{Step 3: Bound the parallel and perpendicular components.}

We now derive lower and upper bounds for the parallel component $\rho_k$ and the perpendicular component $\|\boldsymbol{\epsilon}^{(k)}\|$, respectively.

$\bullet$ \textit{Parallel component $\rho_k$ (lower bound).} We define the normalized advantage magnitude and the normalized score function norm as:

\begin{equation}
    \alpha_k \coloneqq \frac{|\mathbb{E}[|\hat{A}_t| | o_t \in G_k]|}{A_{\max}} \in (0, 1], \qquad
    \beta \coloneqq \frac{\mathbb{E}[\|\nabla_\theta \log \pi_\theta\|]}{G_{\max}} \in (0, 1].
    \label{eq:alpha-beta}
\end{equation}
Note that $\alpha_k$ uses the absolute value of the expected advantage, ensuring $\alpha_k > 0$ for both correct and wrong groups. By definition, we reformulate $\mathbb{E}[|\hat{A}_t| | G_k] = \alpha_k A_{\max}$, $\mathbb{E}[\|\nabla_\theta \log \pi_\theta\|] = \beta G_{\max}$. Furthermore, by Eq.(\ref{eq:weight-bound}), $w_t^{(k)} \geq 1 - \epsilon$ for non-zero weights and Assumption \ref{asmp:direction}: $(1-\eta_t) \geq (1-\eta)$, we can derive the lower bound of $\rho_k$ as follows:
\begin{equation}
    \rho_k \geq |G_k| \cdot (1-\epsilon) \cdot \alpha_k A_{\max} \cdot \beta G_{\max} \cdot (1-\eta).
    \label{eq:rho-lb}
\end{equation}

$\bullet$ \textit{Perpendicular component $\|\boldsymbol{\epsilon}^{(k)}\|$ (upper bound).} We proceed in four sub-steps.

(i) By Jensen's inequality (the norm is convex), we move the norm inside the expectation:
\begin{equation}
    \|\boldsymbol{\epsilon}^{(k)}\| = \left\|\mathbb{E}\left[\sum_{t \in G_k} w_t^{(k)} |\hat{A}_t| \|\nabla_\theta \log \pi_\theta(o_t|s_t)\| \boldsymbol{e}_t\right]\right\| \leq \mathbb{E}\left[\left\|\sum_{t \in G_k} w_t^{(k)} |\hat{A}_t| \|\nabla_\theta \log \pi_\theta(o_t|s_t)\| \boldsymbol{e}_t\right\|\right].
\end{equation}

(ii) By the triangle inequality and the non-negativity of the scalar coefficients $c_t = w_t^{(k)} |\hat{A}_t| \|\nabla_\theta \log \pi_\theta(o_t|s_t)\| \geq 0$:
\begin{equation}
    \left\|\sum_{t \in G_k} c_t \boldsymbol{e}_t\right\| \leq \sum_{t \in G_k} c_t \|\boldsymbol{e}_t\|,
\end{equation}
which gives:
\begin{equation}
    \|\boldsymbol{\epsilon}^{(k)}\| \leq \mathbb{E}\left[\sum_{t \in G_k} w_t^{(k)} |\hat{A}_t| \|\nabla_\theta \log \pi_\theta(o_t|s_t)\| \|\boldsymbol{e}_t\|\right].
\end{equation}

(iii) From Step 1, we have the uniform bound $\|\boldsymbol{e}_t\| \leq \sqrt{2\eta}$.

(iv) For each term in the sum, we apply the individual bounds from  Eq.(\ref{eq:weight-bound}) and Assumptions \ref{asmp:policy} and \ref{asmp:advantage}:
\begin{itemize}[leftmargin=2em, nosep]
    \item $w_t^{(k)} \leq 1+\epsilon$,
    \item $|\hat{A}_t| \leq A_{\max}$,
    \item $\|\nabla_\theta \log \pi_\theta(o_t|s_t)\| \leq G_{\max}$,
    \item $\|\boldsymbol{e}_t\| \leq \sqrt{2\eta}$.
\end{itemize}
Therefore, each term in the sum satisfies:
\begin{equation}
    w_t^{(k)} |\hat{A}_t| \|\nabla_\theta \log \pi_\theta(o_t|s_t)\| \|\boldsymbol{e}_t\| \leq (1+\epsilon) \cdot A_{\max} \cdot G_{\max} \cdot \sqrt{2\eta}.
\end{equation}
Summing over all $|G_k|$ tokens and noting that the right-hand side is a constant (independent of the sampling distribution), the expectation preserves the bound:
\begin{equation}
    \|\boldsymbol{\epsilon}^{(k)}\| \leq |G_k| \cdot (1+\epsilon) \cdot A_{\max} \cdot G_{\max} \cdot \sqrt{2\eta}.
    \label{eq:epsilon-ub}
\end{equation}

\medskip\noindent\textbf{Step 4: Compute the inner product.}

For any two non-trivial groups $G_j$ and $G_k$, substituting the decomposition from Eq.(\ref{eq:decomp}), we get:

\begin{equation}
    \langle \boldsymbol{g}^{(j)}, \boldsymbol{g}^{(k)} \rangle = \langle \rho_j \boldsymbol{d}^* + \boldsymbol{\epsilon}^{(j)},\; \rho_k \boldsymbol{d}^* + \boldsymbol{\epsilon}^{(k)} \rangle.
\end{equation}

Since $\boldsymbol{\epsilon}^{(j)} \perp \boldsymbol{d}^*$ and $\boldsymbol{\epsilon}^{(k)} \perp \boldsymbol{d}^*$, the cross terms vanish, leaving:
\begin{equation}
    \langle \boldsymbol{g}^{(j)}, \boldsymbol{g}^{(k)} \rangle = \rho_j \rho_k + \langle \boldsymbol{\epsilon}^{(j)}, \boldsymbol{\epsilon}^{(k)} \rangle.
\end{equation}

Applying the Cauchy--Schwarz inequality to the perpendicular term, we get:
\begin{equation}
    \langle \boldsymbol{g}^{(j)}, \boldsymbol{g}^{(k)} \rangle \geq \rho_j \rho_k - \|\boldsymbol{\epsilon}^{(j)}\| \cdot \|\boldsymbol{\epsilon}^{(k)}\|.
    \label{eq:inner-lb}
\end{equation}
Substituting the bounds from Eq.(\ref{eq:rho-lb}) and Eq.(\ref{eq:epsilon-ub}), we have the following inequalities:
\begin{align}
    &\rho_j \rho_k \geq |G_j| |G_k| (1-\epsilon)^2 (1-\eta)^2 \alpha_j \alpha_k \beta^2 A_{\max}^2 G_{\max}^2, \nonumber \\
    &\|\boldsymbol{\epsilon}^{(j)}\| \cdot \|\boldsymbol{\epsilon}^{(k)}\| \leq |G_j| |G_k| (1+\epsilon)^2 \cdot 2\eta \cdot A_{\max}^2 G_{\max}^2. 
    \label{eq:eps-eps-ub}
\end{align}
Combining Eq.(\ref{eq:inner-lb}) and Eq.(\ref{eq:eps-eps-ub}), the final inequality is:
\begin{equation}
    \langle \boldsymbol{g}^{(j)}, \boldsymbol{g}^{(k)} \rangle \geq |G_j| |G_k| A_{\max}^2 G_{\max}^2 \Bigl[(1-\epsilon)^2 (1-\eta)^2 \alpha_j \alpha_k \beta^2 - (1+\epsilon)^2 \cdot 2\eta\Bigr].
    \label{eq:inner-final}
\end{equation}

\medskip\noindent\textbf{Step 5: Upper-bound the norm.}

To compute the cosine similarity, we also need upper bounds on the gradient norm. The squared norm of $\bm{g}^{(k)}$ is:
\begin{align}
    \|\bm{g}^{(k)}\|^2
    &= \left\|\E\left[\sum_{t \in G_k} w_t^{(k)}\,\hat{A}_t \nabla_\theta \log \pi_\theta(o_t | s_t)\right]\right\|^2, \notag \\
    &\leq \E\left[\left(\sum_{t \in G_k} w_t^{(k)}\,\hat{A}_t \|\nabla_\theta \log \pi_\theta(o_t | s_t)\|\right)^2\right], \notag \\
    &\leq |G_k| \cdot \E\left[\sum_{t \in G_k} \bigl(w_t^{(k)}\bigr)^2\,\hat{A}_t^2 \|\nabla_\theta \log \pi_\theta(o_t | s_t)\|^2\right].
    \label{eq:norm-ub-sqrt}
\end{align}
where the last step uses Cauchy--Schwarz: $(\sum_i a_i)^2 \leq n \sum_i a_i^2$.

By Eq.(\ref{eq:weight-bound}), $w_t^{(k)} \leq 1+\epsilon$, and Assumptions \ref{asmp:policy} and \ref{asmp:advantage}, $|\hat{A}_t| \leq A_{\max}$ and $\|\nabla_\theta \log \pi_\theta\| \leq G_{\max}$. We have:
\begin{equation}
    \|\bm{g}^{(k)}\|^2 \leq |G_k|^2 \cdot (1 + \epsilon)^2 \cdot A_{\max}^2 \cdot G_{\max}^2.
    \label{eq:norm-ub-final}
\end{equation}
and consequently:
\begin{equation}
    \|\bm{g}^{(k)}\| \leq |G_k|\,(1 + \epsilon)\, A_{\max}\, G_{\max}.
    \label{eq:norm-ub}
\end{equation}

\medskip\noindent\textbf{Step 6: Compute the cosine similarity lower bound.}

The cosine similarity between $\boldsymbol{g}^{(j)}$ and $\boldsymbol{g}^{(k)}$ is defined as:
\begin{equation}
    \operatorname{CosSim}\bigl(\boldsymbol{g}^{(j)}, \boldsymbol{g}^{(k)}\bigr) = \frac{\langle \boldsymbol{g}^{(j)}, \boldsymbol{g}^{(k)} \rangle}{\|\boldsymbol{g}^{(j)}\| \cdot \|\boldsymbol{g}^{(k)}\|}.
\end{equation}
Substituting the inner product lower bound from Eq.(\ref{eq:inner-final}) and the norm upper bound from Eq.(\ref{eq:norm-ub}):
\begin{align}
    \operatorname{CosSim}\bigl(\boldsymbol{g}^{(j)}, \boldsymbol{g}^{(k)}\bigr)
    &\geq \frac{|G_j| |G_k| A_{\max}^2 G_{\max}^2 \Bigl[(1-\epsilon)^2 (1-\eta)^2 \alpha_j \alpha_k \beta^2 - (1+\epsilon)^2 \cdot 2\eta\Bigr]}{|G_j| |G_k| (1+\epsilon)^2 A_{\max}^2 G_{\max}^2}, \nonumber \\
    &= \frac{(1-\epsilon)^2 (1-\eta)^2 \alpha_j \alpha_k \beta^2 - (1+\epsilon)^2 \cdot 2\eta}{(1+\epsilon)^2}.
\end{align}
Splitting into two terms and simplifying:
\begin{equation}
    \operatorname{CosSim}\bigl(\boldsymbol{g}^{(j)}, \boldsymbol{g}^{(k)}\bigr) \;\geq\; \underbrace{\frac{(1-\epsilon)^2}{(1+\epsilon)^2}}_{\displaystyle \coloneqq\;\kappa(\epsilon)} \cdot (1-\eta)^2 \cdot \alpha_j \alpha_k \beta^2 \;-\; 2\eta.
    \label{eq:cosim-final}
\end{equation}
This completes the proof of the main bound. 
\end{proof}

\begin{remark}
    \textbf{Positivity condition}. The bound in Eq.(\ref{eq:cosim-final}) is strictly positive when:
\begin{equation}
    \kappa(\epsilon)(1-\eta)^2 \alpha_j \alpha_k \beta^2 > 2\eta.
\end{equation}

For strong advantage signals ($\alpha_j = \alpha_k = \beta = 1$), the condition simplifies as:
\begin{equation}
    \kappa(\epsilon)(1-\eta)^2 > 2\eta.
\end{equation}

For different $\epsilon$ values, we calculate the critical values of $\eta$ and the corresponding max angular deviations:

\begin{center}
    \begin{tabular}{cccc}
    \toprule
    $\epsilon$ & $\kappa(\epsilon)$ & $\eta_{\text{critical}}$ & Max angular deviation \\
    \midrule
    0.1 & 0.669 & 0.21 & $\arccos(0.79) \approx 38^\circ$ \\
    \midrule
    0.2 & 0.444 & 0.16 & $\arccos(0.84) \approx 33^\circ$ \\
    \midrule
    0.3 & 0.290 & 0.11 & $\arccos(0.89) \approx 27^\circ$ \\
    \bottomrule
\end{tabular}
\end{center}

where the value of $\eta_{\text{critical}}$ is determined by the smaller root of the quadratic equation: $\eta_{\text{critical}} = \frac{\kappa + 1 - \sqrt{2\kappa + 1}}{\kappa}$. One observable phenomenon is that $\eta_{\text{critical}}$ decreases as $\epsilon$ increases. A larger clipping parameter means a wider range of weight variations and thus requires stricter gradient direction consistency between different tokens. For the typical choice $\epsilon = 0.2$, the theorem guarantees strictly positive inter-group cosine similarity as long as $\eta \lesssim 0.16$, meaning each token's advantage-weighted gradient deviates from the reference direction by at most $33^\circ$.

Moreover, this theorem can also explain why modern reinforcement learning algorithms are all based on the PPO-style clipped objective, because through clipping, we can obtain a larger max angular deviation, the tolerance for token gradient direction consistency is higher, thus the training is more stable. 

\end{remark}

\section{Hyperparameters Setting}
\label{sec:hyperparameters_sup}

The full hyperparameters for training our HTPO is shown in Tab.\ref{tab:hyperparameters}.
\begin{table*}[ht]
\centering
\caption{HTPO verl training configuration.}
\label{tab:hyperparameters}
\resizebox{0.6\linewidth}{!} {
\begin{tabular}{lr}
    \toprule 
    \textbf{Parameter} & \textbf{Value} \\
    \midrule 
    trainer.total\_epochs & 1 \\
    data.gen\_batch\_size & 256 \\
    data.train\_batch\_size & 128 \\
    data.max\_prompt\_length & 2048 \\
    data.max\_response\_length & 20480 \\
    algorithm.kl\_ctrl.kl\_coef & 0.0 \\
    algorithm.filter\_groups.enable & True \\
    actor\_rollout\_ref.rollout.n & 16 \\
    actor\_rollout\_ref.rollout.temperature & 1.0 \\
    actor\_rollout\_ref.rollout.top\_p & 1.0 \\
    actor\_rollout\_ref.rollout.top\_k & -1 \\
    actor\_rollout\_ref.rollout.val\_kwargs.top\_p & 0.7 \\
    actor\_rollout\_ref.actor.ppo\_mini\_batch\_size & 32 \\
    actor\_rollout\_ref.actor.loss\_agg\_mode & token-mean \\
    actor\_rollout\_ref.actor.entropy\_coeff & 0.0 \\
    actor\_rollout\_ref.actor.optim.lr & $1e^{-6}$ \\
    actor\_rollout\_ref.actor.optim.weight\_decay & 0.1 \\
    actor\_rollout\_ref.actor.policy\_loss.clip\_entropy\_ratio1 & 0.006 \\
    actor\_rollout\_ref.actor.policy\_loss.clip\_entropy\_ratio2 & 0.02 \\
    actor\_rollout\_ref.actor.policy\_loss.difficulty\_level & 0.5 \\
    actor\_rollout\_ref.actor.clip\_ratio\_low & 0.2 \\
    actor\_rollout\_ref.actor.clip\_ratio\_high & 0.28 \\
    reward\_model.overlong\_buffer.enable & True \\
    reward\_model.overlong\_buffer.len & 4096 \\
    \bottomrule 
\end{tabular}}
\end{table*}

\section{Addition Experimental Results and Analysis}
\label{sec:additional_results}

\subsection{Additional Training Dynamics}
\label{sec:additional_training_dynamics}

Fig.\ref{fig:training_dynamcis_sup1}(a) depicts the mean reward trajectories on the training set, which are critical for evaluating sample efficiency. The plots indicate that all methods, including our proposed HTPO and the baselines, basically converge within 1200 training steps. This demonstrates that the substantial performance improvements achieved by HTPO are not attributed to extended training duration, but rather stem from more effective learning signals derived from our token-level divide-and-conquer mechanism.

Fig.\ref{fig:training_dynamcis_sup1}(b) presents the mean reward trajectories on AIME'24. The observed trends are highly consistent with those in Fig.\ref{fig:training_dynamics}(b) of the main paper: HTPO consistently achieves a higher final reward ceiling than the DAPO baseline, demonstrating the robustness of our performance gains.

Fig.\ref{fig:training_dynamcis_sup1}(c) displays the response length clip ratio, which is the fraction of generated sequences that
reach the maximum token limit. The significantly lower clip ratio, coupled with higher testing reward, further demonstrates that our HTPO is not only superior to GRPO and DAPO in performance, but also more efficient. This indicates that our method does not blindly encourage the model to generate more elaborate and detailed reasoning chains, but rather guides it toward exploring more efficient solving paths.

\begin{figure*}[ht]
    \centering
    \includegraphics[width=1.0\linewidth]{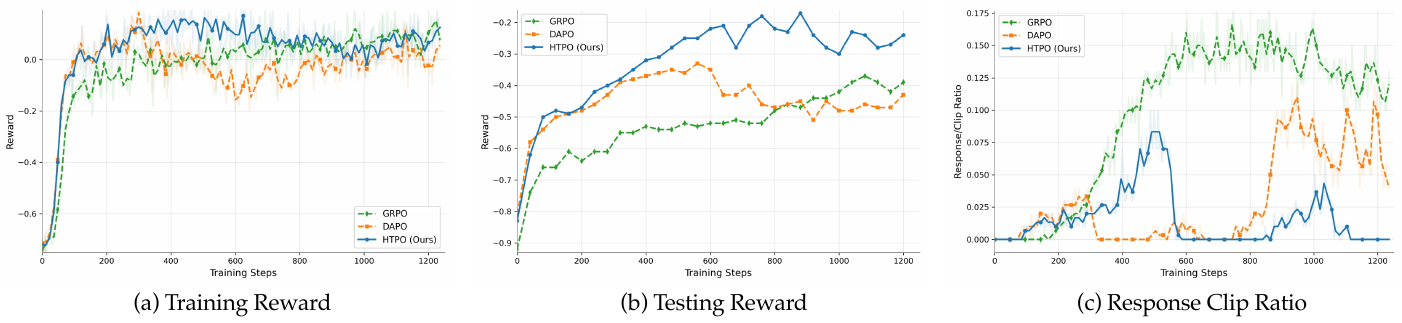}
    \caption{Training dynamics of the Qwen3-8B-Base model on AIME'24. The plots show the evolution of (a) training reward, (b) testing reward, and (c) response length clip ratio. The training reward curves show that all methods converge on the training data by approximately 1200 steps, confirming that the performance gains of our approach are not achieved as the expense of sample efficiency. The testing reward curves further demonstrate the robust and consistent performance improvements of our method. The response clip ratio curves indicate that our method does not blindly enhance exploration but maintains the efficiency of responses.}
    \label{fig:training_dynamcis_sup1}
\end{figure*}

\subsection{Additional Test-Time Scale Results}
\label{sec:additional_test_time_scale}

To further validate the effectiveness of HTPO in scaling test-time computation, we present additional results on four benchmarks: AIME'24, AMC'23, Minerva, and Olympiadbench (Fig.\ref{fig:test_time_scale_sup}). Similar to the trend observed on AIME'25 (Fig.\ref{fig:test_time_scale}), HTPO maintains a consistent performance advantage over GRPO and DAPO across all datasets and sampling budgets. These results underscore HTPO’s robustness and effectiveness in leveraging increased test-time computation to boost model performance.

\begin{figure*}[ht]
    \centering
    \begin{subfigure}[t]{0.48\textwidth}
        \includegraphics[width=\linewidth]{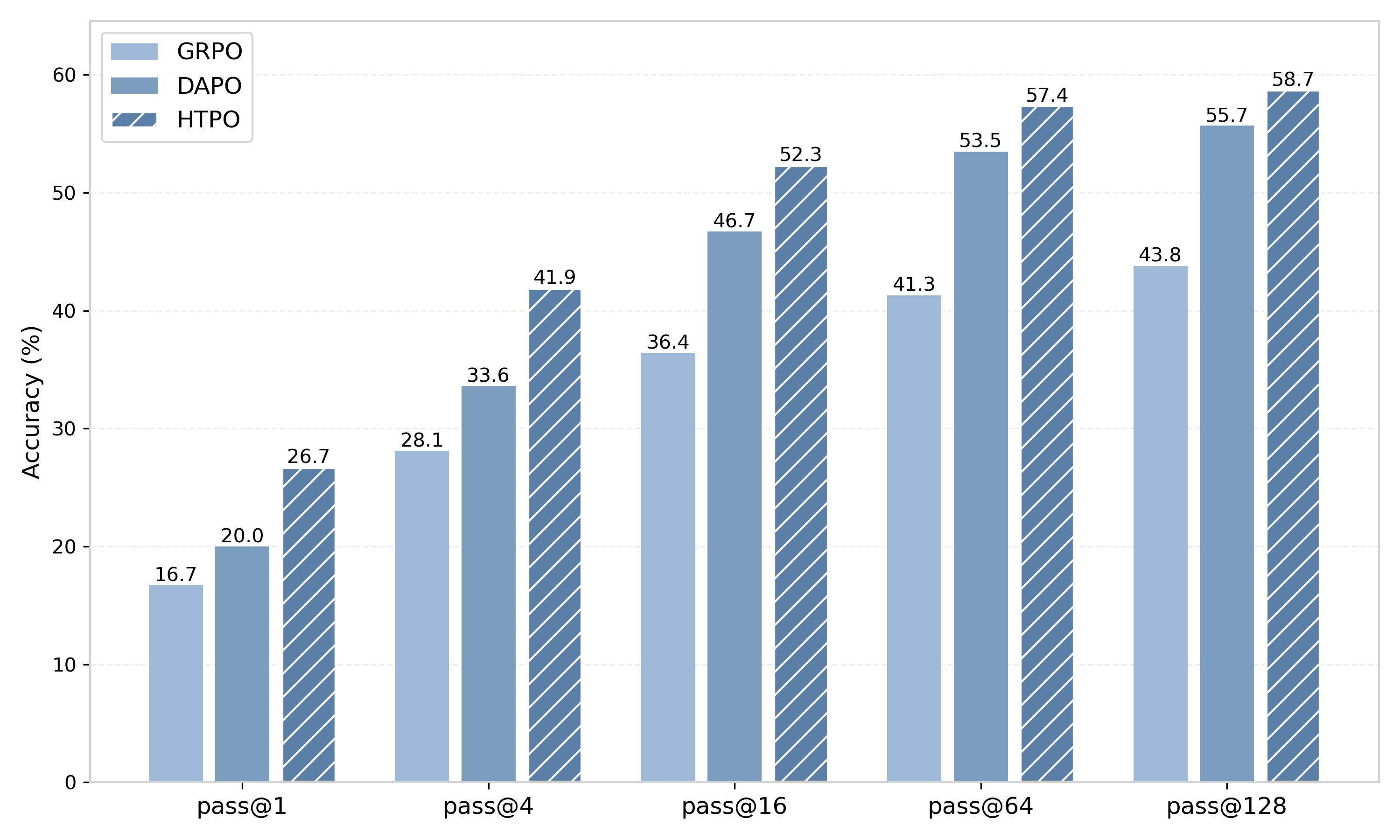}
        \caption{AIME'24}
        \label{fig:scale_aime24}
    \end{subfigure}
    \hfill
    \begin{subfigure}[t]{0.48\textwidth}
        \includegraphics[width=\linewidth]{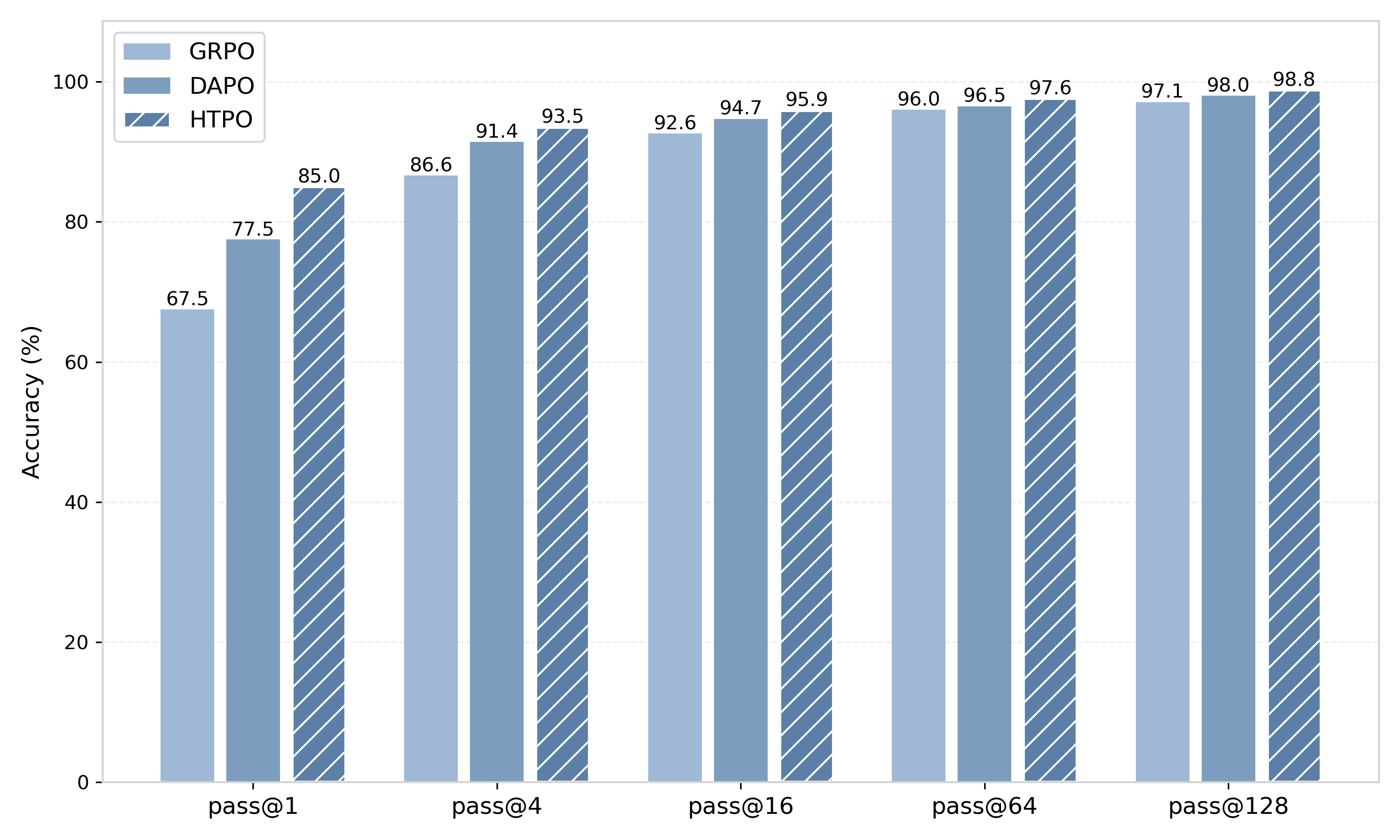} 
        \caption{AMC'23}
        \label{fig:scale_amc23}
    \end{subfigure}
    \begin{subfigure}[t]{0.48\textwidth}
        \includegraphics[width=\linewidth]{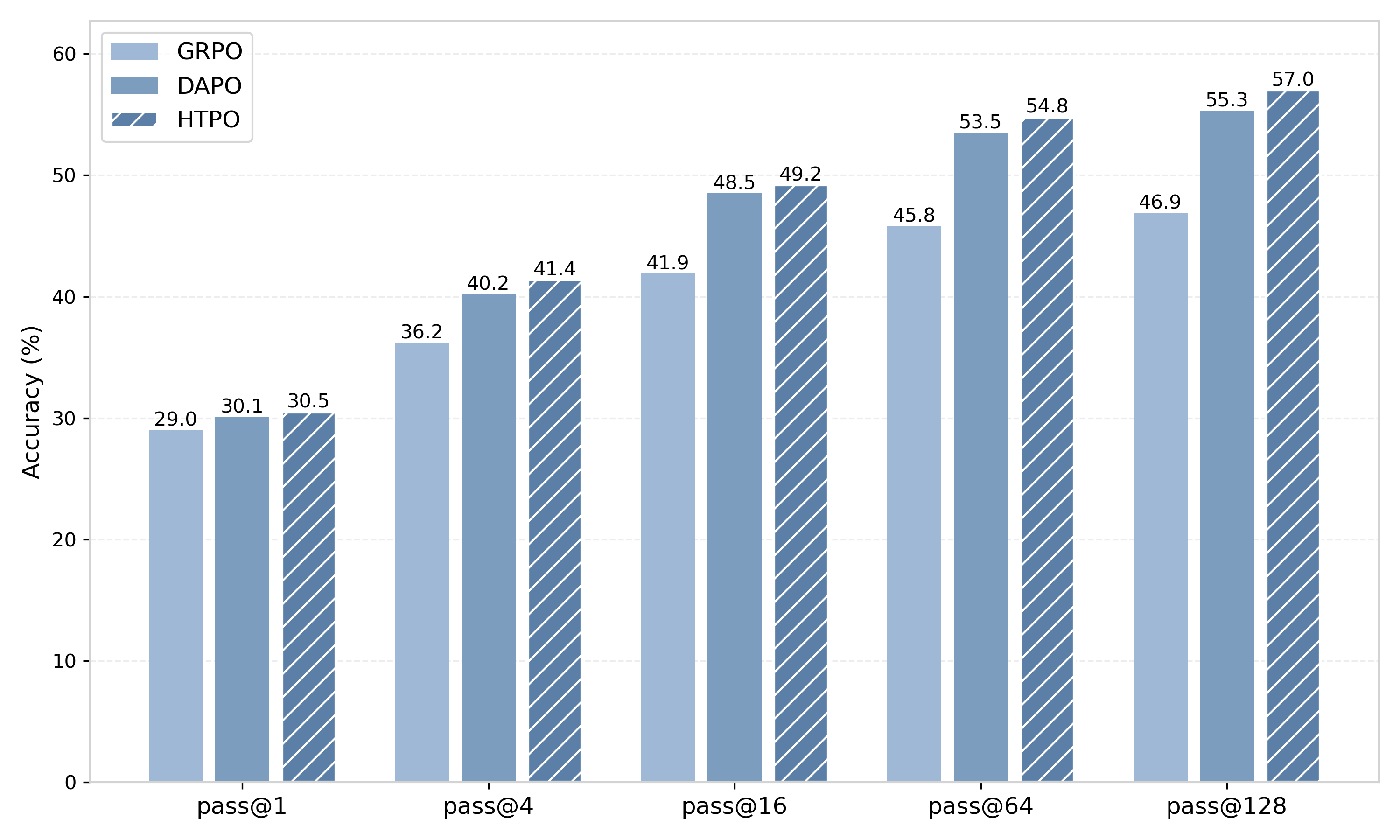}
        \caption{Minerva}
        \label{fig:scale_minerva}
    \end{subfigure}
    \hfill
    \begin{subfigure}[t]{0.48\textwidth}
        \includegraphics[width=\linewidth]{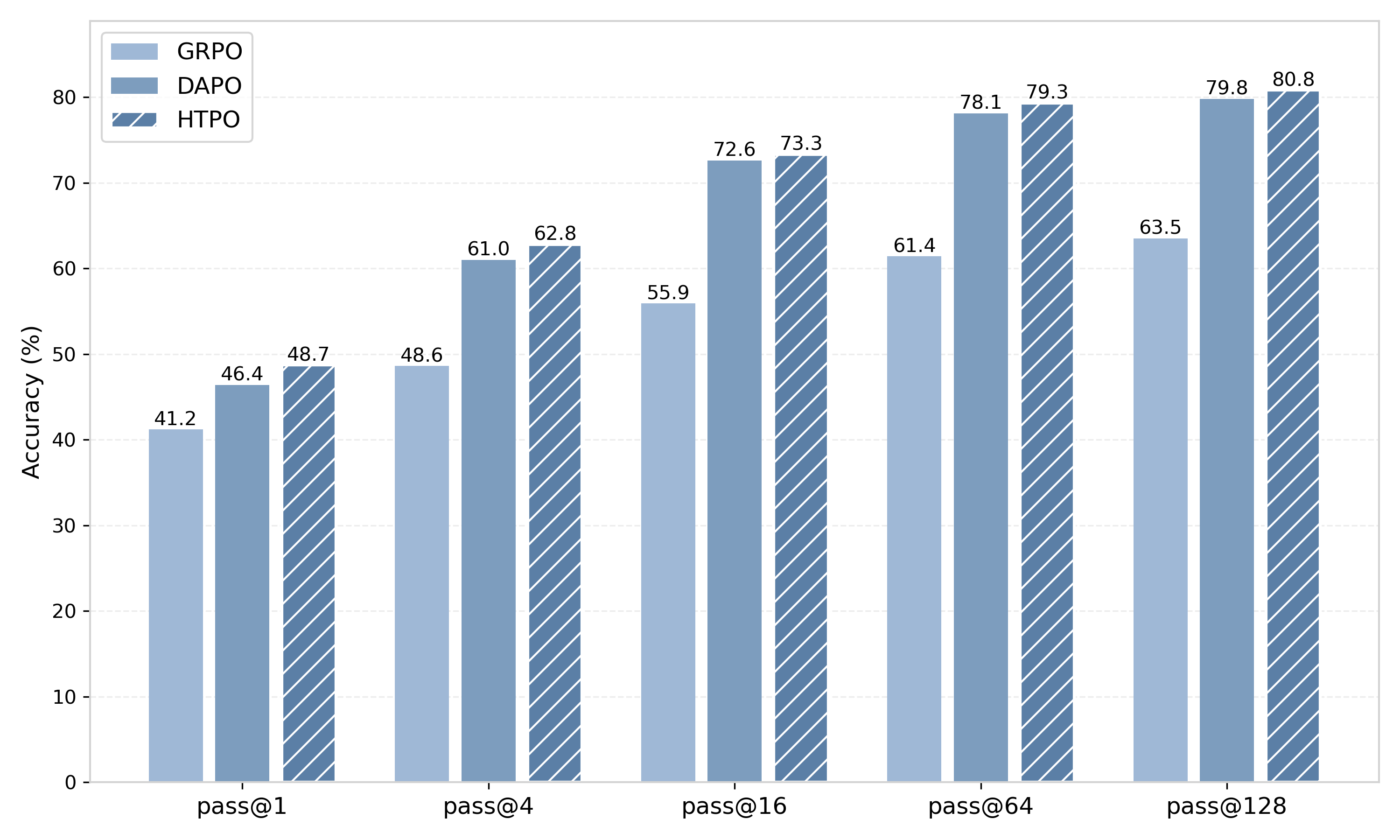} 
        \caption{Olympiadbench}
        \label{fig:scale_olympiadbench}
    \end{subfigure}
    \caption{Impact of scaling test-time computation on diverse reasoning benchmarks. We evaluate the performance of GRPO, DAPO, and HTPO across four additional datasets: (a) AIME’24, (b) AMC’23, (c) Minerva, and (d) Olympiadbench. HTPO consistently outperforms GRPO and DAPO across all benchmarks and sampling budgets, demonstrating its robustness in scaling test-time computation.}
    \label{fig:test_time_scale_sup}
\end{figure*}

\subsection{Understanding HTPO: The Role of Each Group}
\label{sec:group_policy_entropy}

In Tab.\ref{tab:method_ablation}, we demonstrate the effectiveness of our proposed group-specific optimization objectives. To further investigate the specific effects of these objectives and verify their alignment with our design analysis in Sec.\ref{sec:method}, we illustrate the entropy evolution patterns of different groups in Fig.\ref{fig:group_entropy_ablation}. Please note that for each group, we apply our proposed optimization objective to the target group while retaining the original optimization objective for the others (consistent with the setup in Tab.\ref{tab:method_ablation}).

As shown in Fig.\ref{fig:group_entropy_ablation}, entropy increases more significantly in \textbf{Group 1} and \textbf{Group 2}, playing a role in promoting model exploration. In \textbf{Group 4}, \textbf{6}, and \textbf{8}, entropy stabilizes as training progresses, playing a role in preventing over-exploration. These observations align well with our design principles (see Sec.\ref{sec:method}):

$\bullet$ \textbf{Group 1}: We discard the policy gradients of some lowest-entropy tokens to avoid the impact of highly certain tokens on reducing policy entropy.

$\bullet$ \textbf{Group 2}: We enhance the impact of high-entropy tokens on policy learning to guarantee the exploratory capabilities of the model on hard prompts.

$\bullet$ \textbf{Group 4 \& 8}: In wrong answers, we impose heavier penalties on high-entropy tokens to mitigate root errors caused by over-exploration of the model.

$\bullet$ \textbf{Group 6}: We clip a fraction of the highest-entropy tokens from policy gradient updates to avoid overly complex reasoning chains on easy prompts.

The observed entropy patterns in \textbf{Group 1} and \textbf{Group 2} are consistent with our goal of fostering exploration. Similarly, the stabilization seen in \textbf{Group 4}, \textbf{6}, and \textbf{8} also reflects our goal to avoid incorrect responses or unnecessary complexity for easy prompts caused by model over-exploration.

\begin{figure*}[ht]
    \centering
    \includegraphics[width=0.6\linewidth]{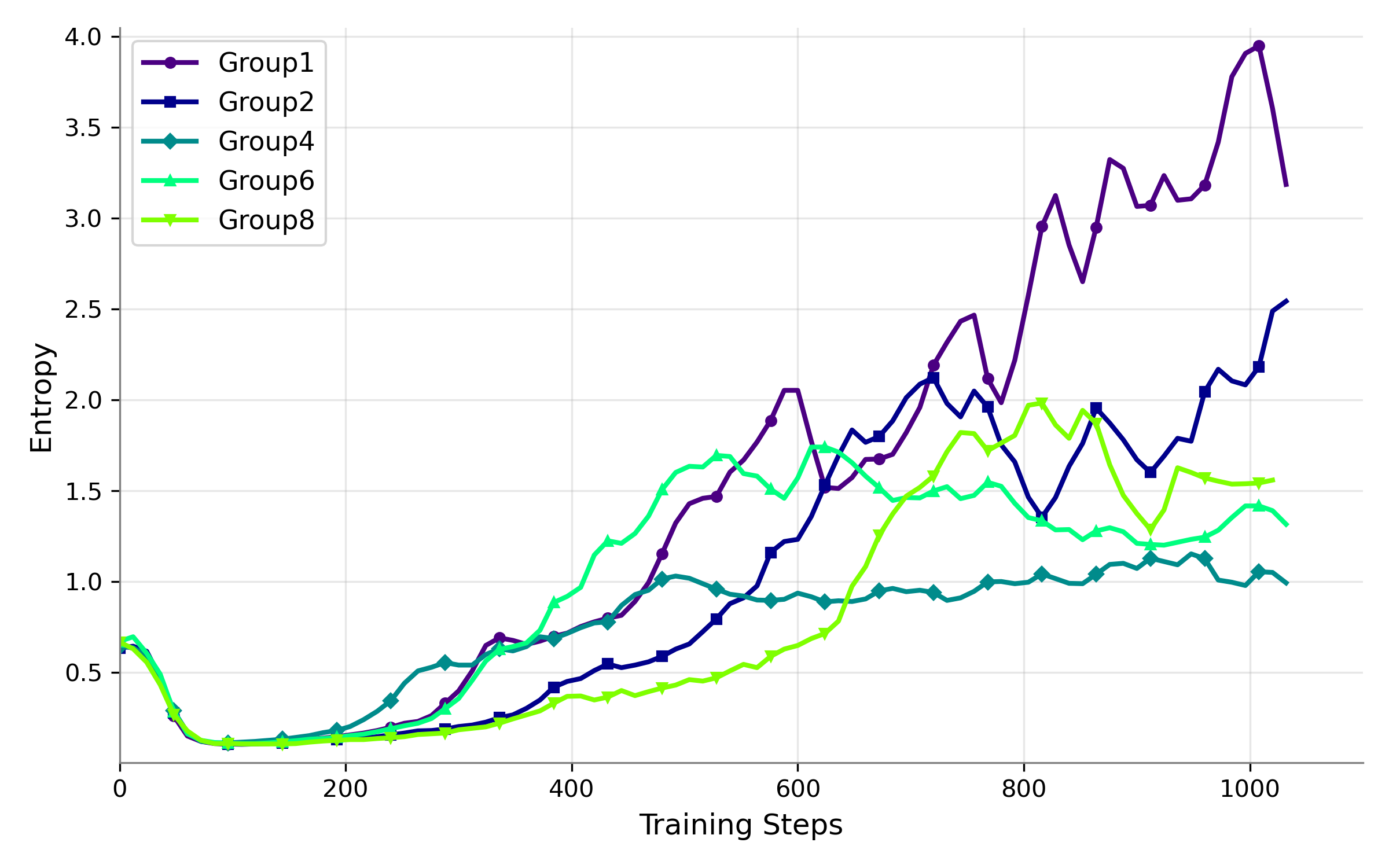}
    \caption{Policy entropy evolution patterns in different token groups. For each group, we apply our proposed optimization objective to the target group while retaining the original optimization objective for the others.}
    \label{fig:group_entropy_ablation}
\end{figure*}

\section{Experimental Argument of Token Entropy Patterns}
\label{sec:token_entropy_analysis}

In this section, we present the patterns of token entropy through both qualitative visualizations and quantitative statistical results. In this way, we aim to more intuitively illustrate the motivation of our method and substantiate the rationality behind the optimization objective designs proposed in the main text (Sec.\ref{sec:method}). To this end, we utilize Qwen3-8B \cite{Qwen3}, one of the most recent and capable reasoning models, to generate responses for prompts from the DAPO-Math-17K dataset \cite{DAPO}, with a decoding temperature of $T = 1.0$. We enforce the model to output chain-of-thought responses for every question. For each prompt, 16 responses are generated, and we collect over $10^6$ tokens from 500 questions. For each token $o_t$, the entropy is computed as $H_t = -\Sigma_{v\in \mathcal{V}}\pi_{\theta_{old}}(v|q,o_{<t}){\rm log}\pi_{\theta_{old}}(v|q,o_{<t})$, where $\mathcal{V}$ represents the token vocabulary. First, we conduct an overall analysis of all the tokens from the perspective of entropy, counting the top 100 tokens with the highest and lowest
average entropy, respectively. Then, in Fig.\ref{fig:word_cloud}, we plot the word cloud maps of these top 100 tokens. Through this intuitive visualization, we can conclude the typical characteristics of token entropy patterns:

1. Tokens with the highest entropy are typically the logical connection or transition words, such as ``however'' and ``unless'' (indicating logical contrasts or shifts), ``hence'' and ``therefore'' (presenting progressive reasoning), or ``since'' and ``because'' (expressing causality). Similarly, tokens like ``suppose'', ``given'' and ``let'' frequently appear in mathematical derivations to introduce assumptions, known conditions, or definitions.

2. By contrast, tokens with the lowest entropy tend to words that just complete the current sentence or sub-words that finish constructing a word, all of which don't directly affect the reasoning logic and exhibit high certainty.

\begin{figure*}[ht]
    \centering
    \begin{subfigure}[t]{0.48\textwidth}
        \includegraphics[width=\linewidth]{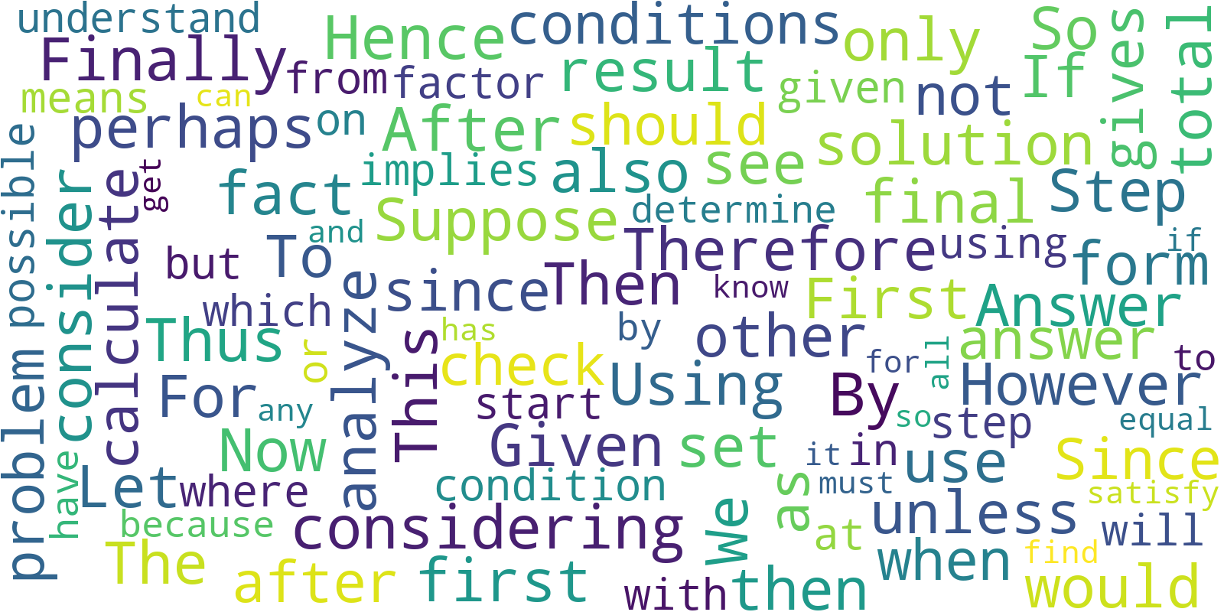}
        \caption{Frequent tokens with the highest average entropy}
        \label{fig:high100}
    \end{subfigure}
    \hfill
    \begin{subfigure}[t]{0.48\textwidth}
        \includegraphics[width=\linewidth]{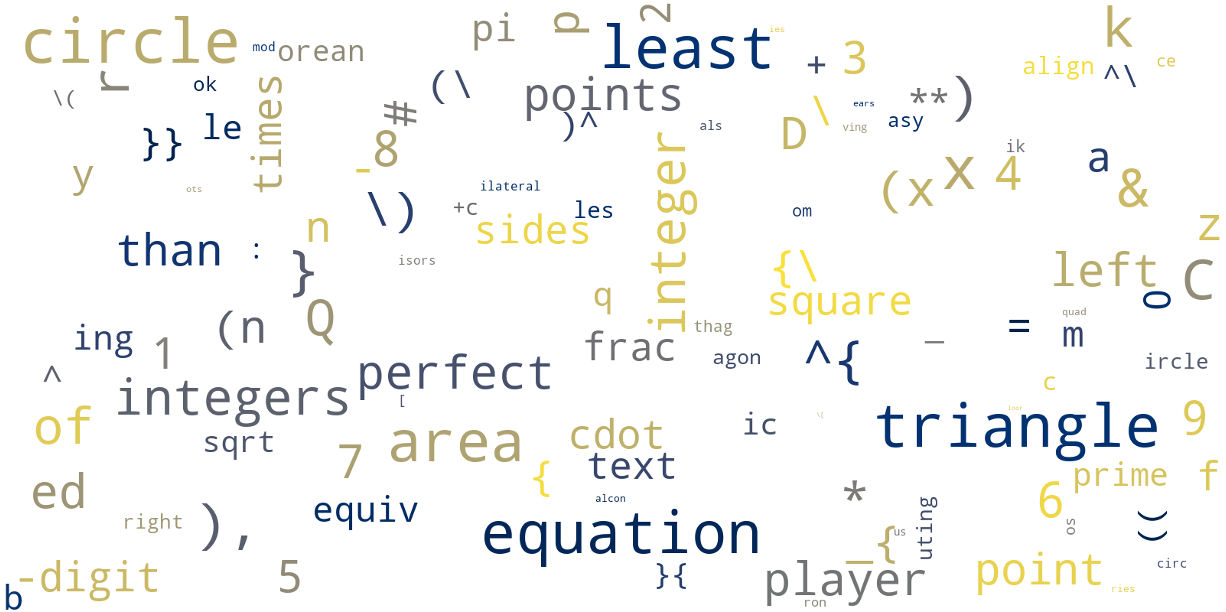} 
        \caption{Frequent tokens with the lowest average entropy}
        \label{fig:low100}
    \end{subfigure}
    \caption{Word cloud maps of the top 100 tokens with the highest and lowest average entropy, respectively. A larger font size indicates a higher average token entropy. Tokens with the highest average entropy typically function as logical connectors or transistors to determine reasoning directions, such as  ``hence'', ``therefore'', ``however'' and ``if'', etc. While tokens with the lowest average entropy are basically the words or sub-words that are not likely to affect reasoning logic in natural language.}
    \label{fig:word_cloud}
\end{figure*}

Our analysis results are also consistent with the conclusion drawn in the previous work \cite{80/20-rule}. In Chain-of-Thought (CoT) reasoning, high-entropy tokens primarily influence the reasoning directions, which are usually associated with the exploration ability and diversity of the model. Whereas low-entropy tokens tend to execute reasoning steps along the established path and primarily guarantee the linguistic coherence, representing the exploitation ability and certainty of the model.

Furthermore, since our method also divides tokens based on prompt difficulty and answer correctness, it is essential to examine the specific patterns of token entropy across these different cases (\emph{i.e.}, hard/easy prompts, correct/wrong answers). To this end, we further provide the token entropy heatmaps for an entire CoT response in Fig.\ref{fig:hard_correct} to \ref{fig:easy_wrong}. By observing these figures, we can identify the following token entropy patterns:

1. Regardless of hard/easy prompts or correct/wrong responses, logical tokens mainly exist in high-entropy tokens, whereas low-entropy tokens usually consist of words that have minimal impact on the reasoning process (this does not imply that there are no logical words, such as ``therefore'', ``since'' and ``but'', in low-entropy tokens, but the quantity is relatively small).

2. Compared to correct responses, the model exhibits higher uncertainty in incorrect
responses. The entropy of high-entropy tokens in wrong responses is overall higher than that of high-entropy tokens in correct ones.

The two token patterns observed above also validate the motivation rationale behind our algorithmic design. Specifically, in correct responses, we should focus more on high-entropy tokens to maintain the model's exploration capability, while effectively leveraging the deterministic knowledge in low-entropy tokens without over-exploitation. In wrong responses, we should impose stronger penalties on high-entropy tokens to avoid the erroneous reasoning paths, while also suppressing low-entropy tokens, as they also have contribution to the wrong answers.

In the following, based on Fig.\ref{fig:hard_correct} to \ref{fig:easy_wrong}, we will provide a case-by-case analysis to demonstrate that the algorithmic designs proposed in the main text are well aligned with the specific token entropy patterns observed in each case. To this end, we can more intuitively support the rationality of our method.

1. As illustrated in Fig.\ref{fig:hard_correct}, for correct responses to hard prompts, high-entropy tokens predominantly correspond to logical connectors, such as ``since'', ``consider'', ``if'', ``which'', and ``thus'', etc. In contrast, low-entropy tokens primarily consist of mathematical expressions that should be deterministic knowledge of the model. Given the complexity of hard prompts, maintaining the model's exploration ability is crucial to facilitate logical transitions and uncover potential solution paths. Consequently, we could pay more attention to high-entropy tokens while avoiding the over-exploitation of low-entropy tokens. In this regard, the optimization objectives proposed in \textbf{Group 1} and \textbf{Group 2} in the main text align closely with the analysis conclusion drawn here.

2. In Fig.\ref{fig:hard_wrong}, compared to Fig.\ref{fig:hard_correct}, the model exhibits higher uncertainty in wrong responses (\emph{i.e.}, the entropy of high-entropy tokens in wrong responses is overall higher than that in correct ones). These high-entropy tokens also contain many logical connection words, such as ``leading'', ``ultimately'', ``clearly'', ``therefore'' and ``thus'', etc, which often tend to trigger subsequent erroneous reasoning steps (\emph{e.g.}, in the second example, the ``Thus \$E[\symbol{"5C} operator name\{area\}] = \symbol{"5C} frac\{1\}\{\symbol{"5C} max(\symbol{"5C} max(a,b)+|a-b|,1\})\$, so the answer is \$4 + 1=\symbol{"5C} boxed\{5\}\$'', where the token ``thus'' introduces a flawed derivation, thereby derailing the final correct answer). Consequently, these high-entropy tokens should receive stronger penalties to effectively suppress incorrect reasoning directions. For low-entropy tokens, as they participate in incorrect reasoning steps, they also should be penalized to prevent the reinforcement of errors. Therefore, the optimization objectives proposed in \textbf{Group 3} and \textbf{Group 4} in the main text are well aligned with the analysis conclusion presented here.

3. In Fig.\ref{fig:easy_correct}, mirroring the patterns observed in hard prompts, in easy prompts, the logical words exhibit a similar preference for appearing as high-entropy tokens. However, unlike hard prompts, for easy prompts, we should encourage exploitation and don't want too complex reasoning chains. Therefore, conversely, we should avoid too many logical transitions and directly learn low-entropy tokens for effective exploitation. To this end, the optimization objectives designed in \textbf{Group 5} and \textbf{Group 6} in the main text are also aligned with the analysis conclusion drawn here.

4. In Fig.\ref{fig:easy_wrong}, consistent with our preceding analysis, high-entropy tokens also tend to dictate the reasoning directions in easy prompts (\emph{e.g.}, the ``therefore'' in the first example, the ``obviously'' and  ``assuming'' in the second). Consequently, to curb these erroneous reasoning paths, we should also, analogous to the strategy for hard prompts/wrong answers, impose stronger suppression on these high-entropy tokens. Moreover, even in easy prompts, the model is still confident about the low-entropy tokens in wrong responses, indicating that the model may have some inherent erroneous biases. Therefore, to prevent the model from persisting in confident errors for easy prompts, it's also crucial to ensure that these erroneous low-entropy tokens can be effectively suppressed. In this regard, the optimization objectives designed in \textbf{Group 7} and \textbf{Group 8} in the main text align closely with the analysis conclusion presented here.

\begin{figure*}[ht]
    \centering
    \includegraphics[width=0.6\linewidth]{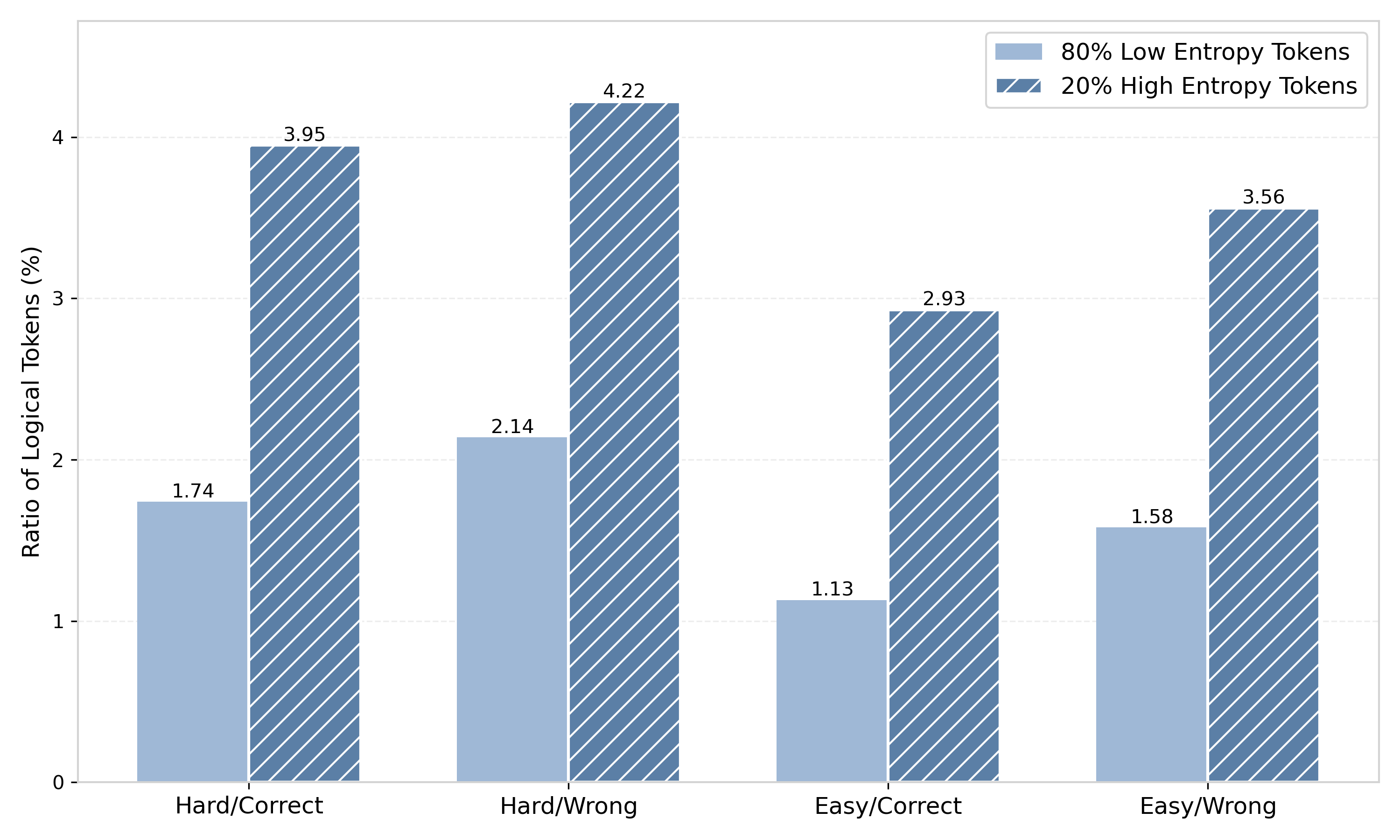}
    \caption{Occurrence ratio of logical words from Fig.\ref{fig:high100} appearing in the high-entropy and low-entropy token groups.}
    \label{fig:logical_token_ratio}
\end{figure*}

Finally, we quantitatively calculate the occurrence rate of logical connectors and transition words from Fig.\ref{fig:high100} appearing in the high-entropy and low-entropy token groups as divided by our method. The results are shown in Fig.\ref{fig:logical_token_ratio}. It can be found that the proportion of logical words appearing in high-entropy tokens is significantly higher than that in low-entropy tokens. This trend holds consistently across hard/easy problems as well as correct/wrong responses. This finding further substantiates the basic idea of our method design: the optimization objective design in high-entropy tokens should be around exploration, whereas the optimization objective design in low-entropy tokens should focus on exploitation.

\begin{figure*}[ht]
    \centering
    \begin{subfigure}[t]{0.7\textwidth}
        \includegraphics[width=\linewidth]{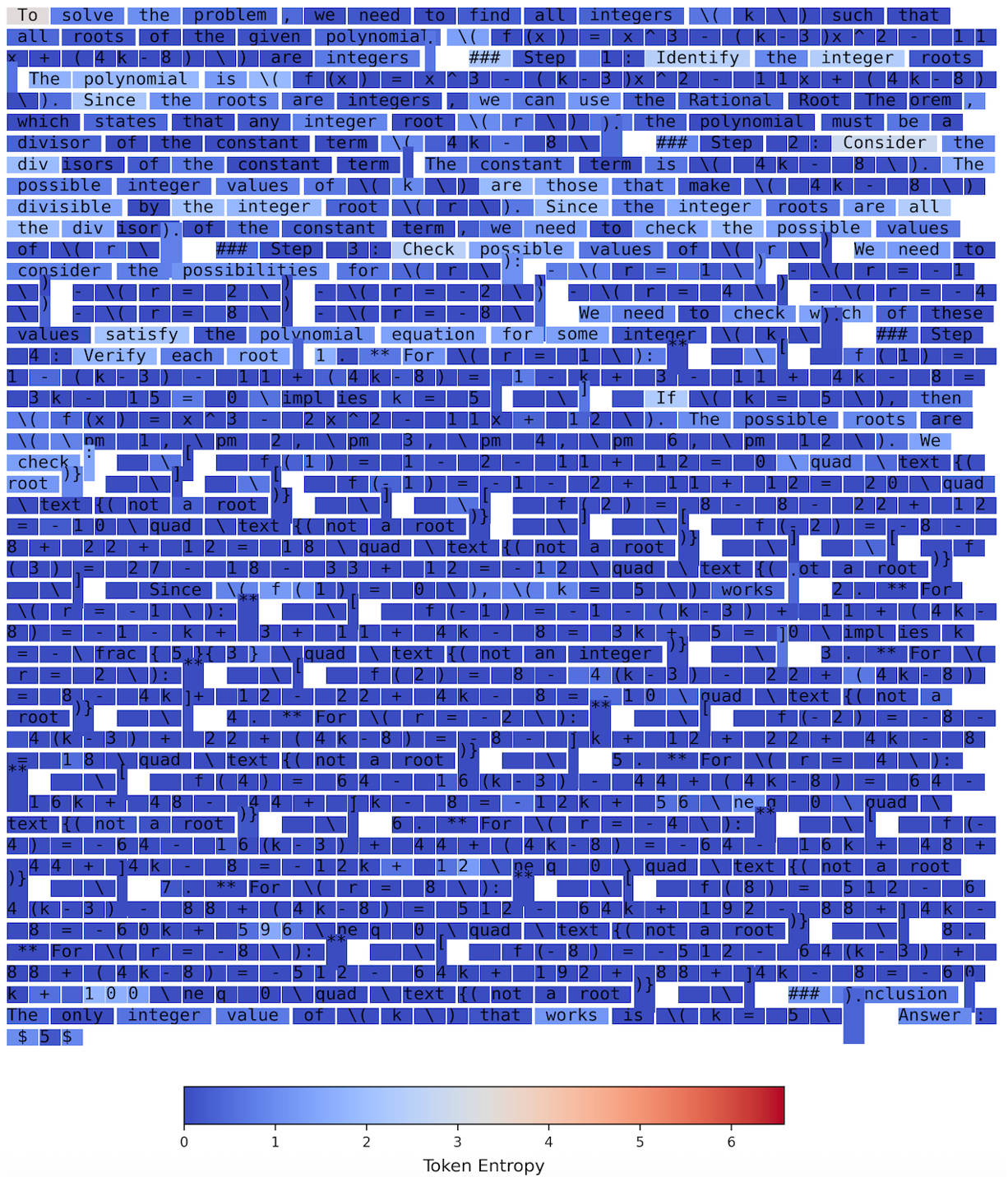}
    \end{subfigure}
    \hfill
    \begin{subfigure}[t]{0.7\textwidth}
        \includegraphics[width=\linewidth]{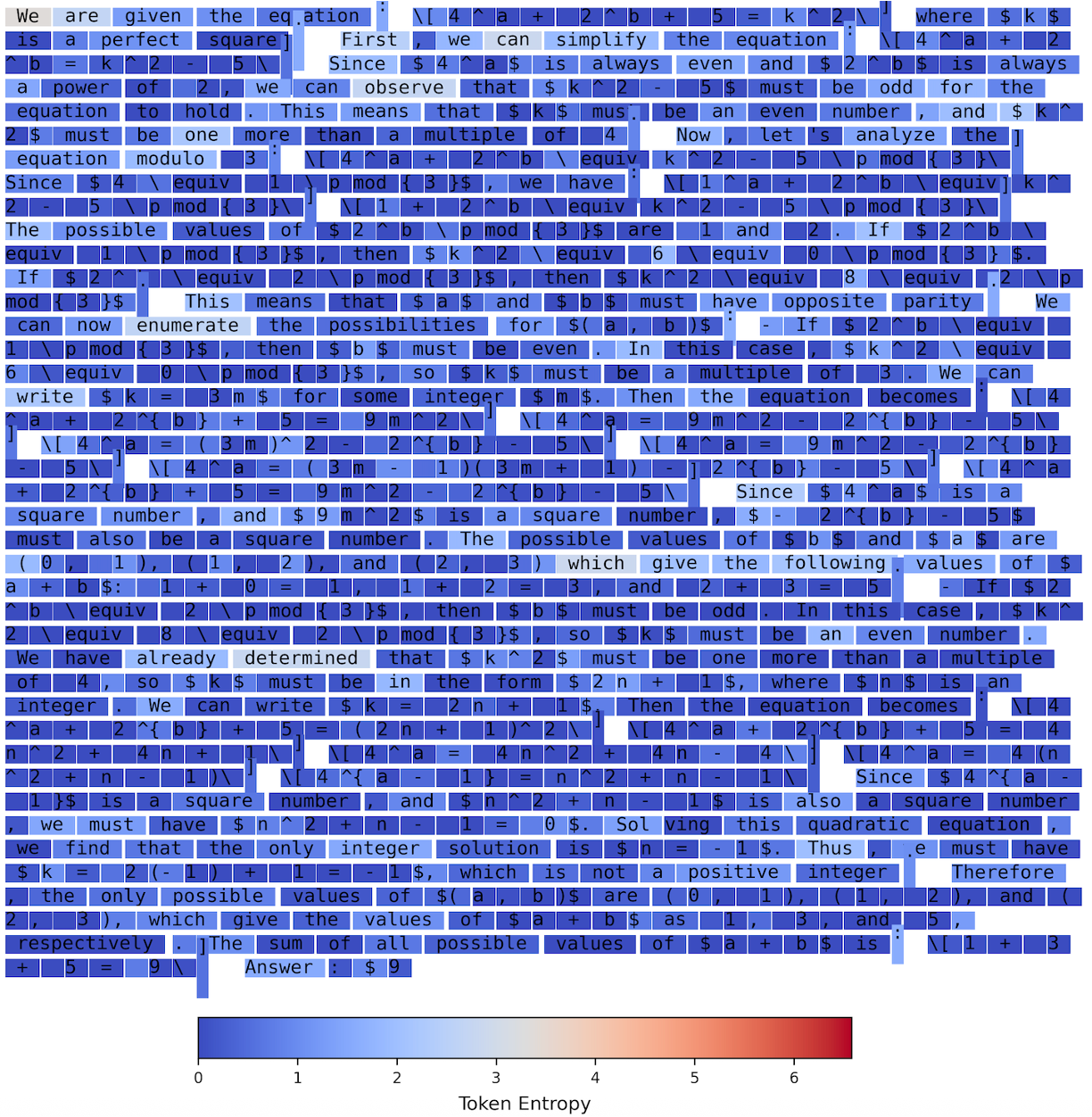}
    \end{subfigure}
    \caption{Visualization of token entropy. These tokens are from the correct responses generated by the model for hard prompts.}
    \label{fig:hard_correct}
\end{figure*}

\begin{figure*}[ht]
    \centering
    \begin{subfigure}[t]{0.7\textwidth}
        \includegraphics[width=\linewidth]{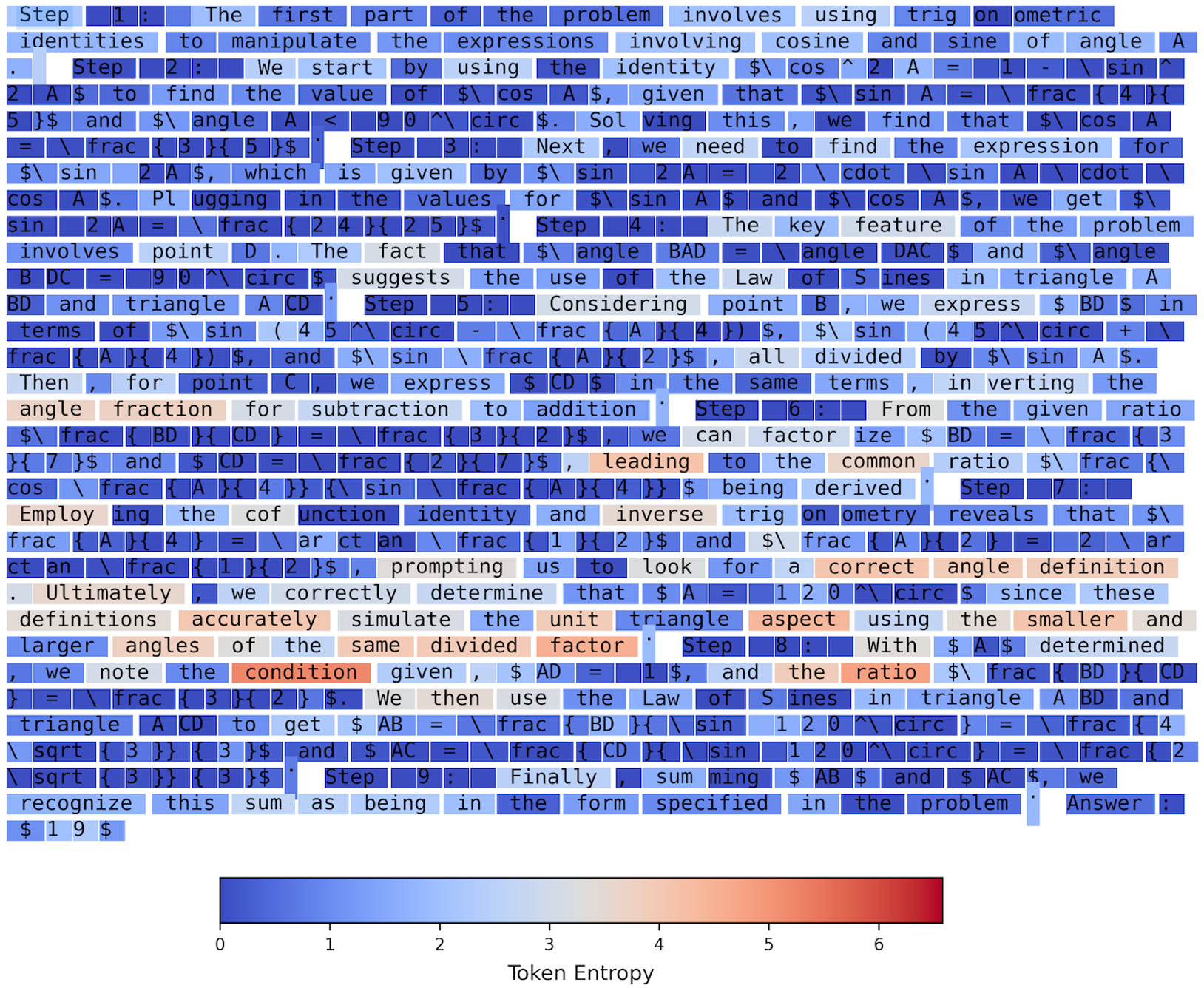}
    \end{subfigure}
    \hfill
    \begin{subfigure}[t]{0.7\textwidth}
        \includegraphics[width=\linewidth]{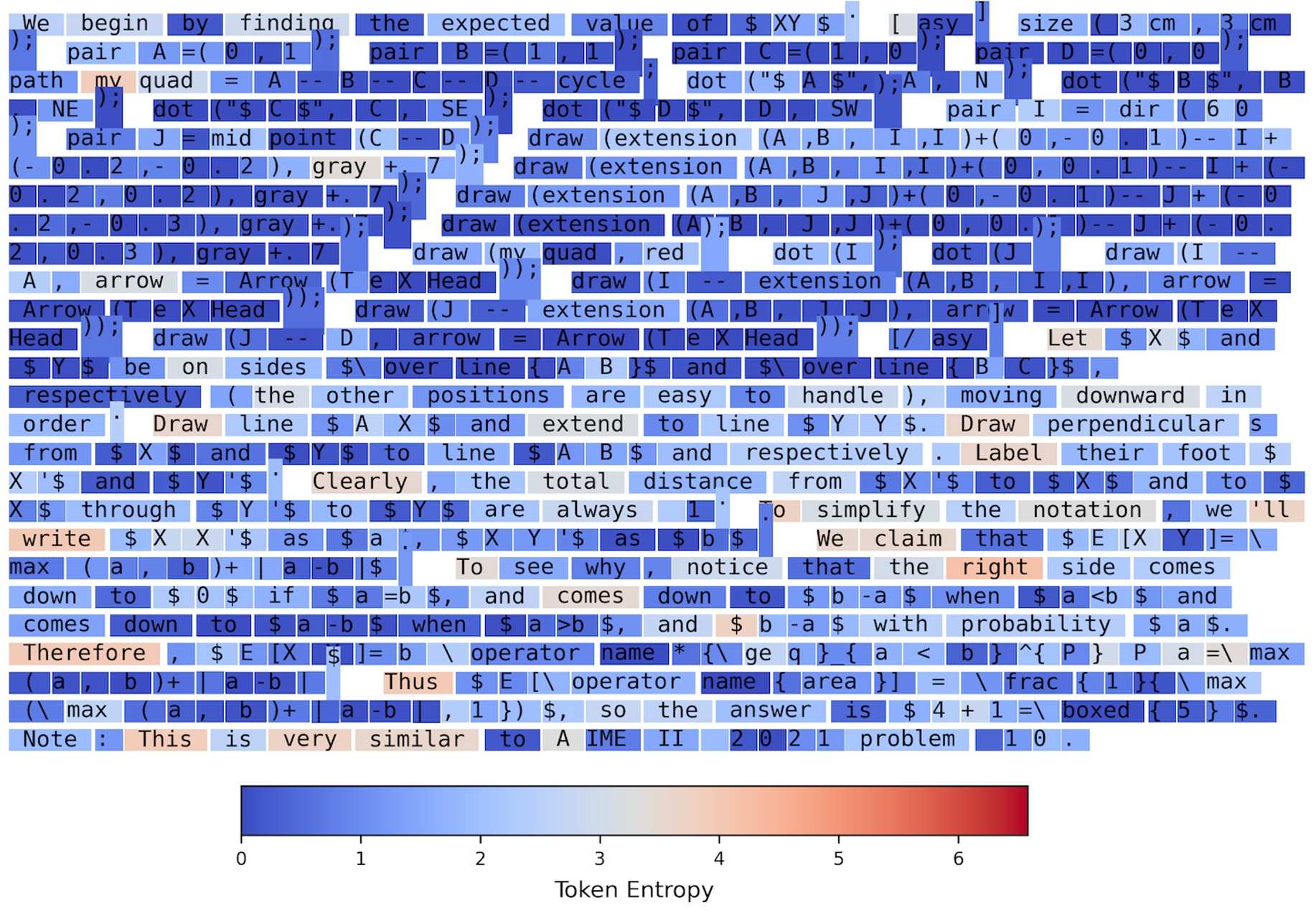}
    \end{subfigure}
    \caption{Visualization of token entropy. These tokens are from the wrong responses generated by the model for hard prompts.}
    \label{fig:hard_wrong}
\end{figure*}

\begin{figure*}[ht]
    \centering
    \begin{subfigure}[t]{0.7\textwidth}
        \includegraphics[width=\linewidth]{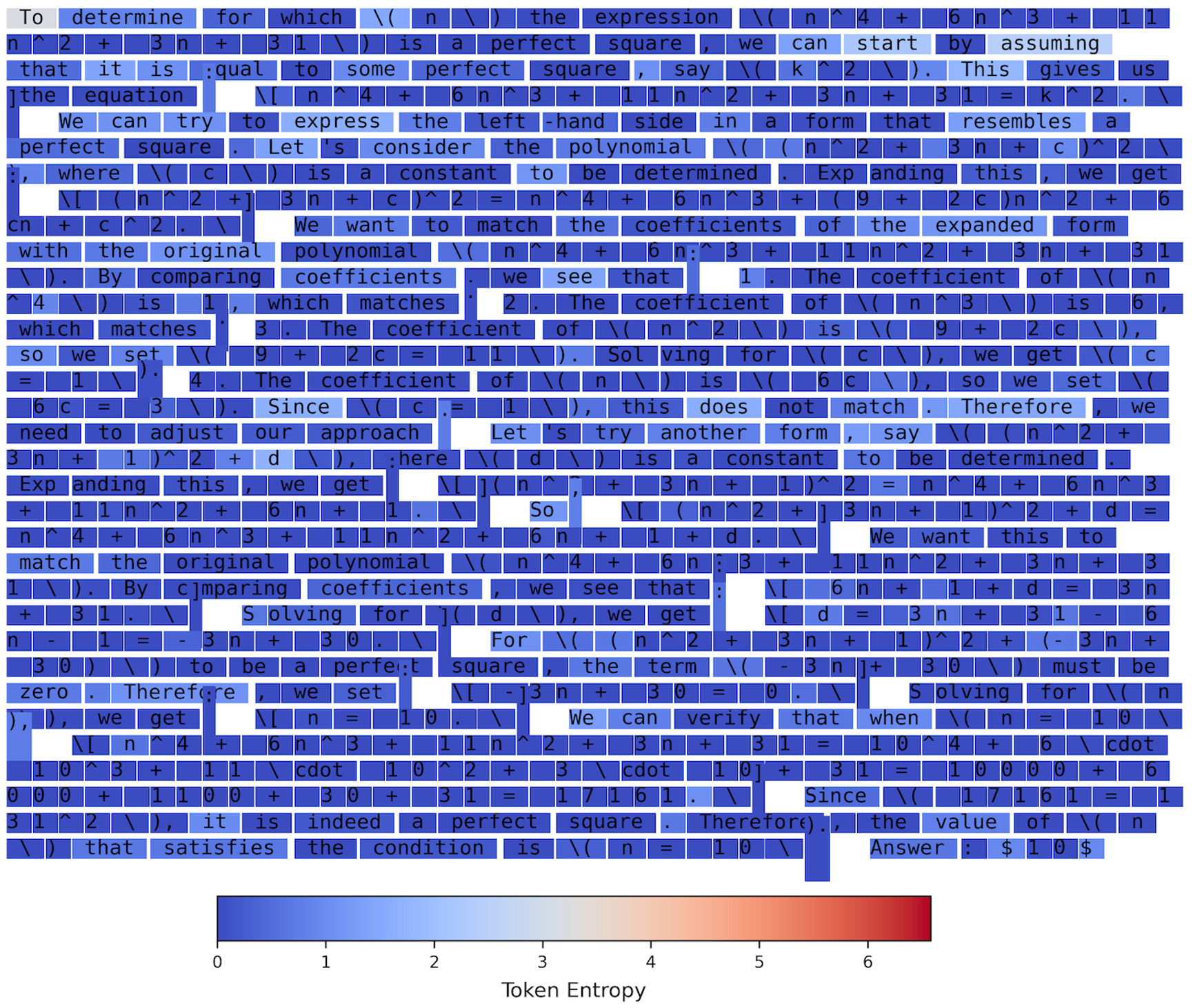}
    \end{subfigure}
    \hfill
    \begin{subfigure}[t]{0.7\textwidth}
        \includegraphics[width=\linewidth]{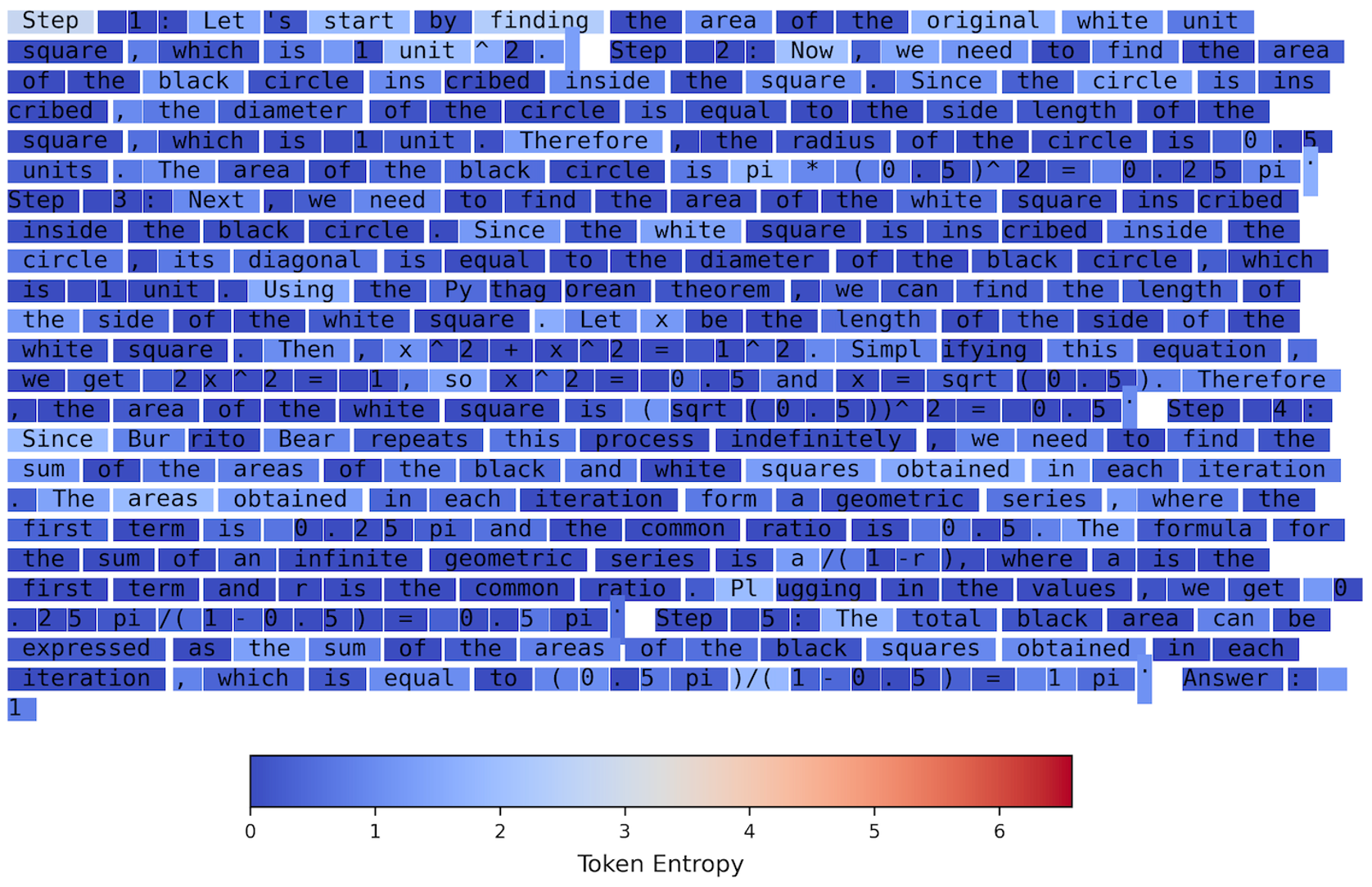}
    \end{subfigure}
    \caption{Visualization of token entropy. These tokens are from the correct responses generated by the model for easy prompts.}
    \label{fig:easy_correct}
\end{figure*}

\begin{figure*}[ht]
    \centering
    \begin{subfigure}[t]{0.7\textwidth}
        \includegraphics[width=\linewidth]{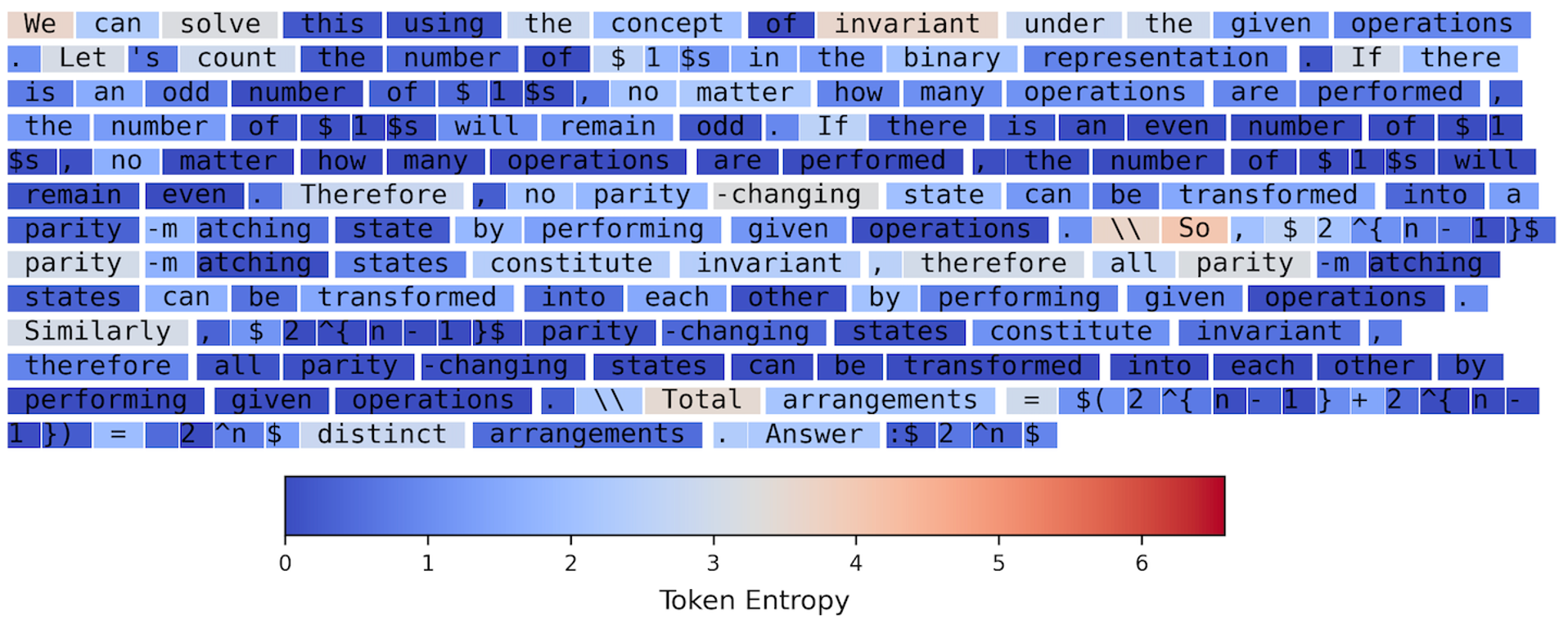}
    \end{subfigure}
    \hfill
    \begin{subfigure}[t]{0.7\textwidth}
        \includegraphics[width=\linewidth]{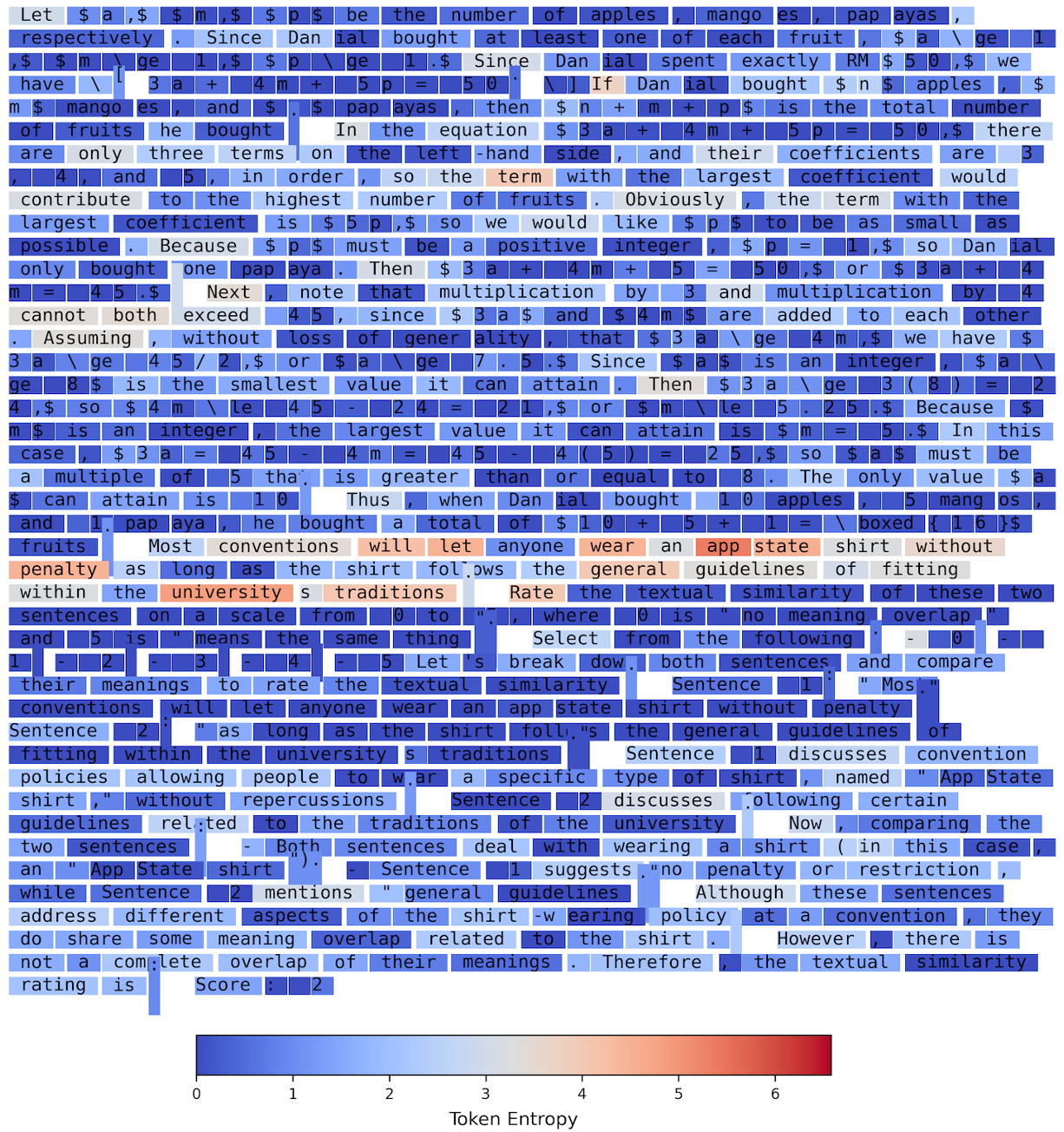}
    \end{subfigure}
    \caption{Visualization of token entropy. These tokens are from the wrong responses generated by the model for easy prompts.}
    \label{fig:easy_wrong}
\end{figure*}


\end{document}